\documentclass[twocolumn]{svjour3}          
\smartqed  
\usepackage{graphicx}
\usepackage{bm}
\usepackage{mathptmx}      
\journalname{Journal of Mathematical Imaging and Vision}
\usepackage[english]{babel}
\usepackage[T1]{fontenc}
\usepackage[latin9]{inputenc}
\usepackage{amsmath}
\usepackage{amsfonts}
\usepackage{amstext}
\usepackage{amscd}
\usepackage{amssymb}
\usepackage{epic}
\usepackage{eepic}
\usepackage{array}
\usepackage{multirow}
\newtheorem{result}{Result}

\usepackage{color}
\usepackage{xcolor}
\usepackage{soul}
\setstcolor{red}

\providecommand{\tabularnewline}{\\}

\usepackage{float}
\floatstyle{ruled}
\newfloat{algorithm}{tbp}{loa}
\providecommand{\algorithmname}{Algorithm}
\floatname{algorithm}{\protect\algorithmname}

\usepackage[breaklinks]{hyperref} 
\hypersetup{colorlinks=true,linkcolor=blue, bookmarks=true}

\usepackage{booktabs}

\begin{document}
\global\long\def\mtype#1{\mathtt{#1}}
\global\long\def\mP{\mtype P}
\global\long\def\mM{\mtype M}
\global\long\def\mR{\mtype R}
\global\long\def\mA{\mtype A}
\global\long\def\mB{\mtype B}
\global\long\def\mr{\mtype r}
\global\long\def\mL{\mtype L}
\global\long\def\mH{\mtype H}
\global\long\def\mQ{\mtype Q}
\global\long\def\Kint{\mtype K} 
\global\long\def\Rot{\mtype R} 
\global\long\def\Real{\mathbb{R}}
\global\long\def\Complex{\mathbb{C}}
\global\long\def\vP{\mathbf{P}}
\global\long\def\bP{\mathbb{P}}
\global\long\def\bQ{\mathbf{Q}}
\global\long\def\bp{\mathbf{p}}
\global\long\def\bq{\mathbf{q}}
\global\long\def\bw{\mathbf{w}}
\global\long\def\bO{\mathbf{O}}
\global\long\def\bo{\mathbf{o}}
\global\long\def\bx{\mathbf{x}}
\global\long\def\bX{\mathbf{X}}
\global\long\def\bI{\mathbf{I}}
\global\long\def\bJ{\mathbf{J}}
\global\long\def\mI{\mtype I} 
\global\long\def\mJ{\mtype J}
\global\long\def\bC{\mathbf{C}} 
\global\long\def\bpi{\boldsymbol{\pi}} 
\global\long\def\bxi{\boldsymbol{\xi}} 
\global\long\def\bchi{\boldsymbol{\chi}} 
\global\long\def\balpha{\boldsymbol{\alpha}}
\global\long\def\bbeta{\boldsymbol{\beta}}
\global\long\def\bomega{\boldsymbol{\omega}}
\global\long\def\br{\mathbf{r}}
\global\long\def\mF{\mtype F}
\global\long\def\bl{\boldsymbol{\ell}}
\global\long\def\be{\mathbf{e}}
\global\long\def\ba{\mathbf{a}}
\global\long\def\bx{\mathbf{x}}
\global\long\def\by{\mathbf{y}}
\global\long\def\bc{\mathbf{c}}
\global\long\def\mC{\mtype C}
\global\long\def\bR{\mathbf{R}}
\global\long\def\mD{\mtype D}
\global\long\def\mK{\mtype K}
\global\long\def\tH{\widetilde{\mH}}
{} \global\long\def\bu{\mathbf{u}}
\global\long\def\bv{\mathbf{v}}

\global\long\def\pinf{\pi_{\infty}}
\global\long\def\bpinf{\boldsymbol{\pi}_{\infty}}
\global\long\def\xinf{\mathbf{x_{\infty}}}
\global\long\def\linf{\mathbf{l}_{\infty}}
\global\long\def\Oinf{\mathtt{\Omega}_{\infty}}
\global\long\def\Qinf{\mathtt{Q}_{\infty}^{\ast}}
\global\long\def\mQa{\mathtt{Q}^{\ast}}
\global\long\def\iac{\omega}
\global\long\def\diac{\omega^{\ast}}
\global\long\def\Sinf{\Sigma_{\infty}}

\global\long\def\ovec{\operatorname{vec}}
\global\long\def\rank{\operatorname{rank}}

\global\long\def\rev#1{{\color{blue}{#1}}}

\title{Autocalibration with the Minimum Number of Cameras\\ with Known
Pixel Shape}



\author{Jos\'e I. Ronda \and Antonio Vald\'es \and Guillermo Gallego
\thanks{This work has been partially supported by the Ministerio de Econom\'ia
y Competitividad of the Spanish Government under project TEC2010-20412
(Enhanced 3DTV). G. Gallego is supported by the Marie Curie-COFUND
Programme of the EU, Seventh Framework Programme.
}}


\institute{J.I. Ronda, G. Gallego \at
              Grupo de Tratamiento de Im\'agenes, Universidad Polit\'ecnica de Madrid, 28040 Madrid, Spain\\
              \email{\{jir,ggb\}@gti.ssr.upm.es}
           \and
           A. Vald\'es \at
	          Dep. de Geometr\'ia y Topolog\'ia, Universidad Complutense de Madrid, 28040 Madrid, Spain\\
              \email{Antonio\_Valdes@mat.ucm.es}
}

\date{Received: 12 December 2012 / Accepted: 2 January 2014}

\maketitle

\begin{abstract}
  In 3D reconstruction, the recovery of the calibration parameters of
  the cameras is paramount since it provides metric information about
  the observed scene, e.g., measures of angles and ratios of
  distances.  Autocalibration enables the estimation of the camera
  parameters without using a calibration device, but by enforcing
  simple constraints on the camera parameters.  In the absence of
  information about the internal camera parameters such as the focal
  length and the principal point, the knowledge of the camera pixel
  shape is usually the only available constraint.  Given a projective
  reconstruction of a rigid scene, we address the problem of the
  autocalibration of a minimal set of cameras with known pixel shape
  and otherwise arbitrarily varying intrinsic and extrinsic
  parameters.  We propose an algorithm that only requires 5 cameras
  (the theoretical minimum), thus halving the number of cameras
  required by previous algorithms based on the same constraint. To
  this purpose, we introduce as our basic geometric tool the six-line
  conic variety (SLCV), consisting in the set of planes intersecting
  six given lines of 3D space in points of a conic.  We show that the
  set of solutions of the Euclidean upgrading problem for three
  cameras with known pixel shape can be parameterized in a
  computationally efficient way. This parameterization
    is then used to solve autocalibration from five or more cameras,
    reducing the three-dimensional search space to a two-dimensional one.
  We provide experiments with real images showing the good performance
  of the technique.  
\end{abstract}

\section{Introduction}
\label{intro}


Three-dimensional reconstructions from images are often obtained with
calibrated cameras, i.e., cameras whose parameters have been
previously computed using calibration objects in a controlled
environment~\cite[p.~201]{YiMaLibro}.  Unfortunately, in many cases
such conditions are not available, e.g., when the images have been
acquired with non-specialized equipment or taken with a different
initial purpose. To obtain 3D reconstructions without knowledge
  of the scene content and with partial knowledge of the
  camera parameters, camera autocalibration algorithms are needed.


Camera autocalibration comprises a family of techniques that
  apply to different scenarios depending on the \emph{a priori}
  information (i.e., constraints) of the internal camera parameters. In the first works on
  autocalibration~\cite{Maybank-Faugeras,Faugeras92TE}, the assumption
  was that the internal camera parameters were unknown but constant.
  Later, other techniques were developed to deal with different
  assumptions on the focal length, the principal point, the skew or the
  aspect ratio (see~\cite[
Ch. 19]{Hartley-Zisserman}).
  The more \emph{a priori} information one can incorporate in an
  autocalibration method, the better the results are expected to be.
  Therefore, there is no decisive solution to autocalibration in all
  situations.

A common framework for autocalibration was provided by the
concept of geometric
stratification~\cite{Luong1996,Faugeras95stratificationof}.  This
technique splits the camera calibration and scene reconstruction into
three steps: first, the recovery of a projective reconstruction, i.e.,
a 3D scene and a set of cameras differing from the actual ones in a
spatial homography.  Second, the obtainment of an affine
  reconstruction (differing from the actual 3D scene in an affine
  transformation) by finding the location of the plane at
infinity~\cite{Pollefeys}.  Finally, the upgrading to a Euclidean
  reconstruction (differing from the actual scene in a similarity) by
localizing the absolute conic at infinity or any equivalent geometric
object.  General references for the subject
are~\cite{Hartley-Zisserman,Faugeras-Luong}, where an extensive
bibliography can be found.  A review of camera self-calibration 
techniques is also presented in~\cite{SurveySelfCalib2003}.

~In order to upgrade a projective reconstruction to an affine or a
Euclidean one in the absence of knowledge about the scene, some data
about the camera parameters must be
available.  For example, in the originally addressed autocalibration
problem, this additional piece of data was the constancy of the camera
intrinsic parameters~\cite{Maybank-Faugeras,Faugeras92TE}, resulting,
for each camera pair, in a couple of
polynomial 
equations in the camera parameters known as Kruppa equations.
The instability problems of these equations have been studied
in~\cite{Sturm00}.
Another constraint that has been studied is that in which the principal
point is assumed to be known. 
Then the dual absolute quadric~\cite{Triggs97}, a geometric object 
which encapsulates the information of both the plane at infinity 
(needed for affine upgrading) and the absolute conic (required for Euclidean upgrading), 
can be found by solving a set of homogeneous linear equations~\cite{Seo2000}.


In this paper we address the problem of autocalibration in its less
restrictive setting in practice: cameras with arbitrarily varying
parameters with the exception of the pixel shape, which is assumed to
be known.
This is an important scenario since the pixel shape is
  unaffected by changes in focus and zoom.  It can be easily seen
(see Sect.~\ref{sec:Background}) that this constraint is equivalent
to having cameras with square pixels. The possibility of employing
this constraint was studied in~\cite{HeydenAstrom}, where an algorithm
based on bundle adjustment~\cite{Triggs99Bundle} was
considered. Algorithms based on this restriction have also been
proposed~\cite{Ponce2,Valdes04,Valdes06} that result in a set of
linear equations, but with the drawback of requiring 10 or more
cameras.  These algorithms are inspired by the geometric observation
that, from the optical center of each square-pixel camera, two lines
can be identified in the projective reconstruction that must intersect
the absolute conic. The absolute quadratic complex (AQC) encodes the
set of lines intersecting this conic by means of a quadric in a higher
dimensional space ($\bP^5$),
which is the natural space containing Pl\"ucker coordinates of lines.
The AQC, being represented by a homogeneous symmetric $6\times 6$
matrix satisfying a linear constraint, depends linearly on $21-1-1=19$
non-homogeneous parameters, which explains the need for such a large
number of 10 views.

However, an informal parameter count reveals that far fewer cameras
are theoretically sufficient. In fact, our main unknown is a space
homography, which depends on 15 parameters. Being our target
reconstruction defined up to a similarity, which depends on 7
parameters, there are $15-7=8$ unknowns left to be found, which
correspond to the degrees of freedom (dof) needed to determine the
plane at infinity (3 dof) and the absolute conic within it (5
dof). Knowing camera skew and aspect ratio amounts to two equations
per camera and thus at least 4 cameras should be given in order to
solve the problem. Given the non-linear nature of these equations,
multiple solutions can be expected and so 5 cameras should be the
minimum required to obtain, generically, a unique solution
(see~\cite[Table 19.3]{Hartley-Zisserman},~\cite{Pollefeys99b}).


The main contribution of this paper is a technique to obtain a
Euclidean reconstruction for an arbitrary number of cameras equal or
above the theoretical minimum, which is~5, using exclusively the pixel
shape restriction. The basic geometric idea of this
  paper consists in the characterization of the candidate planes at
  infinity as those that intersect the isotropic lines of the cameras
  in points of a conic. In fact, this geometric approach was already
  pointed out in the late 19th
  century~\cite{Finsterwalder,Sturm11}.
  Finsterwalder
  showed that from multiple images of a rigid scene, projective
  reconstruction is possible. He showed that for unknown focal lengths
  and principal points, one may back-project all image cyclic points
  to 3D and find the plane at infinity as the one which cuts all those
  3D lines in points which lie on a single conic. He also explains
  that 4 cameras should be the minimal case but that multiple
  solutions will exist. However, no algorithm was provided.
In this paper, we provide such an algorithm. The
geometric object that will be employed for this
purpose is the variety of conics intersecting six given spatial lines
simultaneously, which will be termed the six-lines conic variety
(SLCV). The SLCV as a geometrical entity has been studied
in~\cite{schroecker}, although our treatment is independent and
self-contained. In this paper we are interested in the SLCV given by
the absolute conic at infinity. 

\begin{table*}
\caption{\label{tab:comparison_AQC_SLCV}Comparison of the Absolute Quadratic Complex (AQC) and the Six-Line
Conic Variety (SLCV) approaches.}
\begin{centering}
\begin{tabular}{lll}
\toprule
\multicolumn{1}{l}{} & \textbf{AQC} & \textbf{SLCV}\\
\midrule
\emph{Common features} & \multicolumn{2}{l}{}\\[0.3ex]
Basic geometric object & \multicolumn{2}{l}{Isotropic lines through the optical centers of the cameras.}\\
 
Intrinsic parameter assumption & \multicolumn{2}{l}{Known pixel shape (skew and aspect ratio).}\\
\cmidrule{1-3}
\emph{Differences} & & \\[0.3ex]
Type of algorithm & Linear (solution of a  & Non-linear (bidimensional search \\
 &  homogeneous system). &  using second degree equations).\\[0.5ex]
Required number of cameras & $\geq10$ & $\geq5$\\[0.3ex]
Optimal w.r.t. number of cameras & No & Yes\\[0.3ex]
Geometric object used & Quadric of $\bP^{5}$ & Algebraic surface of $\bP^{3\ast}$ (set of planes of 3-space)\\[0.3ex]
Degree of the geometric object & 2 & 5\\[0.3ex]
Geometric meaning & Lines intersecting the absolute conic. & Planes with conics intersecting all the isotropic lines.\\[0.3ex]
Integration with scene knowledge & Pairs of orthogonal lines. & Parallel lines. Points at infinity.\\
\bottomrule
\end{tabular}
\par\end{centering}
\end{table*}

The SLCV for six lines in generic position can be identified with a
surface of $\bP^{3*}$ (i.e., the projective space given by the planes
of 3D space) of degree 8. We prove that this degree reduces to 5 in the
case of the three pairs of isotropic lines of three finite square-pixel
cameras. We show that the fifth-degree SLCV~ has three singularities of
multiplicity three, given by the three principal planes of the
cameras. This result is used in Algorithm~\ref{alg:Algorithm1Parametrization} 
to generate a bidimensional parameterization of the candidate planes at
infinity compatible with three square-pixel cameras.  This
parameterization, together with the additional data given by another
two or more square-pixel cameras permits to identify the true plane at
infinity through a two-dimensional optimization process, leading to
Algorithm~\ref{alg:Algorithm2Upgrading}. 
However, the technique could as well use other additional data such as some scene constraints.
Table~\ref{tab:comparison_AQC_SLCV} summarizes the similarities and differences 
between the SLCV and the AQC and the algorithms built upon them.

Experiments with real images for the autocalibration of scenes with 5
or more cameras with square pixels and otherwise varying parameters
are provided, showing the good performance of the proposed technique
compared to other autocalibration methods. In the absence of knowledge
about the principal point of the cameras, the SLCV algorithm turns out
to be a feasible
  approach to solve the autocalibration problem, not requiring
  a previous initialization with an approximate solution, in the minimal case
of 5 cameras up to the case of 9 cameras.  For 10 or more cameras, the results are
similar to those of the AQC algorithm.

The paper is organized as follows: The basic background for the
problem is briefly recalled in Sect.~\ref{sec:Background}.
Section~\ref{sec:SLCV} presents the SLCV along with the basic
algebraic geometry tools required for its definition and analysis as
well as our main theoretical results. The algorithms motivated by
these results are presented in Sect.~\ref{sec:Algorithms} and the
corresponding experiments are shown in Sect~\ref{sec:Experiments}.
A comparison with other search-based autocalibration 
algorithms is discussed in Sect.~\ref{sec:discussion}. 
Conclusions of the paper are found in Sect.~\ref{sec:conclusions}.
An advance of some of the results of this paper appeared in the
conference paper~\cite{Carballeira}.

\section{Camera model and preliminary problem analysis}
\label{sec:Background}

We suppose that the cameras can be described using the pinhole camera
model~\cite{Hartley-Zisserman}, which is defined by the optical
center $\bC$ and the projection plane endowed with an affine coordinate
system given by the pixel structure of the camera.

The equations of the projection are linear when expressed in
  homogeneous coordinates, which are defined as follows. Any non-zero
  vector $(U,V,W)^\top$ proportional to $(u,v,1)^\top$ are the homogeneous
  coordinates of the point with usual coordinates
  $(u,v)^\top$.
  The set of all 3-vectors, considering them equal if they are
  proportional, is the projective space $\bP^2$.
  If we suppose the coordinates are Euclidean, the elements of $\bP^2$
  with non-zero last coordinate constitute homogeneous coordinates of
  points of the plane, whereas those elements with vanishing last
  coordinate are identified with points at infinity in a Euclidean
  reference.  The definition extends in a straightforward manner to
  spaces of arbitrary dimensions.

The equations of the projection are of the form $\bx\sim\mP\bX$ where~
the~ symbol $\sim$ represents~ vector~ proportionality,
$\mP=\mK(\mR\,|\,-\mR\widetilde{\bC})$ is the projection matrix of the camera, 
$\bX=(X,Y,Z,T)^{\top}$ are
the homogeneous Euclidean coordinates of a 3D point, $\bx=(x,y,z)^{\top}$
are the homogeneous coordinates of its projection, $\mR$ is a rotation
matrix, $\widetilde{\bC}$ are the usual Euclidean coordinates
of the optical center, and $\mK$ is the intrinsic parameter matrix,
given by 
\[
\mK=\begin{pmatrix}fm_{x} & -fm_{x}\cot\theta & u_{0}\\
0 & fm_{y}/\sin\theta & v_{0}\\
0 & 0 & 1
\end{pmatrix},
\]
 where $f$ is the focal length, $m_{x}$ and $m_{y}$ are the number
of pixels per distance unit in image coordinates in the $x$ and $y$
directions, $\theta$ is the skew angle and $(u_{0},v_{0})$ is the
principal point.

If the camera aspect ratio $\tau=m_{y}/m_{x}$ and the skew angle
$\theta$ are known, the affine coordinate transformation 
given in homogeneous coordinates by matrix
\[
\mA=\begin{pmatrix}\tau & \cos\theta & 0\\
0 & \sin\theta & 0\\
0 & 0 & 1
\end{pmatrix}
\]
 of the image plane permits to assume that the intrinsic parameter
matrix has the form 
\begin{equation}
\Kint=\begin{pmatrix}\alpha & 0 & u'_{0}\\
0 & \alpha & v'_{0}\\
0 & 0 & 1
\end{pmatrix},\label{K}
\end{equation}
 which is the intrinsic parameter matrix of a square-pixel camera,
i.e., one for which $\tau=1$ and $\theta=\pi/2$.

The back-projected lines of cyclic points at infinity ${\bf I}=(1,i,0)^{\top}$
and $\bar{{\bf I}}=(1,-i,0)^{\top}$ are the \emph{isotropic lines}
of the camera. These lines intersect the absolute conic, for if $\bX=(X,Y,Z,0)^{\top}$
is the intersection of one of these two lines with the plane at infinity,
we have 
\[
(1,\pm i,0)^{\top}\sim\mP\bX=\Kint\Rot(X,Y,Z)^{\top},
\]
 so that 
\[
(X,Y,Z)^{\top}\sim\Rot^{\top}\Kint^{-1}(1,\pm i,0)^{\top},
\]
 and then 
\begin{align*}
X^{2}+Y^{2}+Z^{2} & =(X,Y,Z)(X,Y,Z)^{\top}\\
 & =(1,\pm i,0)\Kint^{-\top}\Rot\Rot^{\top}\Kint^{-1}(1,\pm i,0)^{\top}\\
 & =\begin{pmatrix}1 & \pm i\end{pmatrix}\begin{pmatrix}\alpha^{-2} & 0\\
0 & \alpha^{-2}
\end{pmatrix}\begin{pmatrix}1\\
\pm i
\end{pmatrix}=0.
\end{align*}

This equation can be expressed as 
\begin{equation}
\bI^{\top}\iac\bI=0,\qquad\bar{\bI}^{\top}\iac\bar{\bI}=0,\label{eq:square-pixel-IAC}
\end{equation}
 which is the square-pixel condition in terms of the image of the
absolute conic (IAC) $\bomega=(\Kint\Kint^{\top})^{-1}$.

We recall here that it is possible to obtain
a projective calibration from image point correspondences only in uncalibrated images. 
This means that, given a sufficient number of projected points $\bx_{ij}$
obtained with $N_{c}\geq2$ cameras, we can obtain a set of matrices
$\hat{\mP}_{i}$ and a set of point coordinates $\hat{\bX}_{j}$ such
that $\bx_{ij}\sim\hat{\mP}_{i}\hat{\bX}_{j}$, where $\hat{\mP}_{i}=\mP_{i}\mH^{-1}$
and $\hat{\bX}_{j}=\mH\bX_{j}$ for some non-singular $4\times4$
matrix $\mH$.
Projective reconstruction from feature correspondences is a mature topic in computer vision.
It was mostly developed during the 90's stemming from the extension of
the stereo reconstruction method of Longuet-Higgins~\cite{LonguetHiggins} to the case of 
uncalibrated images~\cite{Faugeras,HartleyCVPR92} (two views),~\cite{QuanTrifocal94,Hartley94Invariants} (three views),~\cite{Mohr93cvpr,Triggs95thegeometry} (multiple views), etc.
Nowadays there are excellent reference books collecting the contributions of many researchers in this field; 
see, for example~\cite{Hartley-Zisserman,Faugeras-Luong,YiMaLibro}.

\emph{Euclidean calibration} can be defined as the obtainment of a
matrix $\mH$ changing the projective coordinates of a given projective
calibration to some Euclidean coordinate system. It is well-known
that Euclidean calibration up to a scale factor is equivalent to the
recovery of the absolute conic at infinity $\Omega_{\infty}$
(see~\cite[p.~272]{Hartley-Zisserman},~\cite[\S
1.18]{Faugeras-Luong}), 
or any equivalent object such as the dual
absolute quadric (DAQ)~\cite{Triggs97}.
To motivate our approach, let us check the possibility of addressing
this autocalibration problem using the DAQ.
The DAQ is given by the planes tangent to the absolute conic and it is
algebraically defined as a rank-three projective mapping
$\mQ^\ast_\infty:\bP^{3\ast}\to\bP^3$. Its matrix it is
related to the matrix of the dual IAC (DIAC) by
\begin{equation}
\label{eq:DIAC-DAQ}
  \omega^\ast \sim \mP\mQ^\ast_\infty \mP^\top,
\end{equation}
where $\mP$ is the matrix of the considered projective camera. If the camera has
square pixels, it is known that the matrix of the DIAC is of the form
\begin{displaymath}
  \omega^\ast\sim
  \begin{pmatrix}
    \alpha^2 + x_0^2 & x_0 y_0 & x_0\\
    x_0y_0 & \alpha^2 + y_0^2 & y_0 \\
    x_0 & y_0 & 1
  \end{pmatrix}.
\end{displaymath}
This leads to the equations
\begin{equation}
\label{eq:quadraticDIAC}
  \begin{split}
  \omega^\ast_{12}\omega^\ast_{33}-\omega^\ast_{13}\omega^\ast_{23}&=0\\    
\omega^\ast_{33}\omega^\ast_{11}-\left(\omega^\ast_{13}\right)^2&=
\omega^\ast_{33}\omega^\ast_{22}-\left(\omega^\ast_{23}\right)^2.
  \end{split}
\end{equation}
Using relation~(\ref{eq:DIAC-DAQ}), Eqs.~(\ref{eq:quadraticDIAC})
turn out to be two quadratic equations in the coefficients of
$\mQ^\ast_\infty$ obtained for each camera $\mP$. The DAQ
$\mQ^\ast_\infty$ is homogeneous and symmetric, so it is described by
$9$ parameters. Therefore four cameras will lead to $8$ quadratic
equations which, together with the quartic constraint
$\det\mQ^\ast_\infty=0$ would determine a discrete number of possible
solutions. One more camera would allow to decide which of these
solutions is the correct one.
 
Given the large number of variables and the nonlinear nature of the
equations, this approach does not seem very promising from a practical
point of view. 
As we will see, the alternative we propose in this paper reduces the
problem to a 5th-degree equations in three variables, which allows a two dimensional parameterization of the
possible planes at infinity from which we will be able to do an
efficient search of those compatible with the intrinsic geometry of
the cameras.

\section{The six line conic variety}
\label{sec:SLCV}

Let us suppose we have a projective calibration of three square-pixel
cameras (or, as explained before, cameras with known pixel shape)
and let $\mP_{i}$ be their camera matrices and $\text{\ensuremath{\bC}}_{i}$
their optical centers, $i=1,2,3$. Let us denote by $l_{i},\bar{l}_{i}$ the
isotropic lines of camera $i$. The plane at infinity $\bpinf$ will
intersect these lines in points of the absolute conic. Therefore the
planes $\pi$ candidates to be the plane at infinity are those intersecting
the isotropic lines in points of a conic (see Fig.~\ref{fig:SixLinesAndConic}).
We are going to see that these planes are given by a 5th-degree algebraic equation in 
their coordinates, $\bpi=(u_{1},u_{2},u_{3},u_{4})^{\top}$.
To obtain this equation, some mathematical
preliminaries are needed.

\begin{figure}
\centering \setlength{\unitlength}{0.00063333in}
\begingroup\makeatletter\ifx\SetFigFont\undefined%
\gdef\SetFigFont#1#2#3#4#5{%
  \reset@font\fontsize{#1}{#2pt}%
  \fontfamily{#3}\fontseries{#4}\fontshape{#5}%
  \selectfont}%
\fi\endgroup%
{\renewcommand{\dashlinestretch}{30}
\begin{picture}(4524,2746)(0,-10)
\allinethickness{1.000pt}%
\texture{8101010 10000000 444444 44000000 11101 11000000 444444 44000000 
	101010 10000000 444444 44000000 10101 1000000 444444 44000000 
	101010 10000000 444444 44000000 11101 11000000 444444 44000000 
	101010 10000000 444444 44000000 10101 1000000 444444 44000000 }
\shade\path(12,924)(4216,924)(4512,2249)
	(648,2249)(12,924)
\path(12,924)(4216,924)(4512,2249)
	(648,2249)(12,924)
\put(1113,1413){\blacken\ellipse{48}{48}}
\put(1113,1413){\ellipse{48}{48}}
\put(1403,1822){\blacken\ellipse{48}{48}}
\put(1403,1822){\ellipse{48}{48}}
\put(2223,1856){\blacken\ellipse{48}{48}}
\put(2223,1856){\ellipse{48}{48}}
\put(2652,1844){\blacken\ellipse{48}{48}}
\put(2652,1844){\ellipse{48}{48}}
\put(3931,1707){\blacken\ellipse{48}{48}}
\put(3931,1707){\ellipse{48}{48}}
\put(3791,1453){\blacken\ellipse{48}{48}}
\put(3791,1453){\ellipse{48}{48}}
\put(1271,1057){\blacken\ellipse{38}{38}}
\put(1271,1057){\ellipse{38}{38}}
\put(2365,1452){\blacken\ellipse{38}{38}}
\put(2365,1452){\ellipse{38}{38}}
\put(3487,2203){\blacken\ellipse{48}{48}}
\put(3487,2203){\ellipse{48}{48}}
\path(684,2383)(1113,1413)
\path(2000,2542)(2219,1868)
\drawline(3231,2462)(3231,2462)
\path(2654,1850)(3089,2436)
\path(3345,2571)(3787,1455)
\path(3284,2429)(3925,1716)
\path(3284,2429)(3925,1716)
\dashline{60.000}(1394,1831)(1249,926)
\path(1243,917)(1098,12)
\path(1537,2719)(1392,1814)
\path(1344,911)(1550,435)
\dashline{60.000}(2228,1858)(2538,927)
\path(2546,912)(2701,452)
\dashline{60.000}(1984,916)(2654,1845)
\path(1796,684)(1984,918)
\dashline{60.000}(1120,1402)(1326,926)
\dashline{60.000}(3793,1449)(3983,974)(3986,967)
\dashline{60.000}(3937,1705)(4261,1347)
\put(2331,1599){\ellipse{3550}{524}}
\put(324,1008){\makebox(0,0)[lb]{{\SetFigFont{12}{14.4}{\rmdefault}{\mddefault}{\updefault}$\bpi_\infty$}}}
\put(1523,1406){\makebox(0,0)[lb]{{\SetFigFont{12}{14.4}{\rmdefault}{\mddefault}{\updefault}$\Omega_\infty$}}}
\end{picture}
} \caption{Illustration of the incidence relations between the isotropic lines
of three cameras, the plane at infinity and the absolute conic.}
\label{fig:SixLinesAndConic} 
\end{figure}
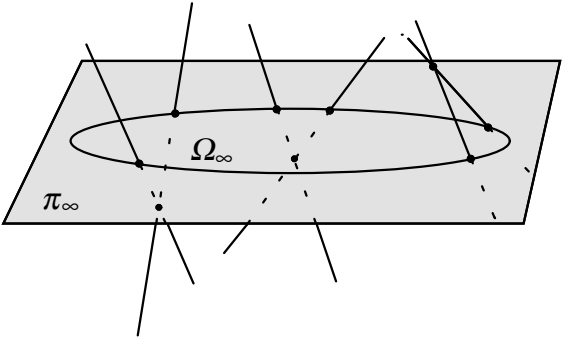

\subsection{The equation of the six-line conic variety}
\label{subsec:plucker_coordinates}
We recall, see e.g.~\cite[p.~70]{Hartley-Zisserman}, that lines of
3-space are in one-to-one correspondence with non-null singular
antisymmetric $4\times4$ matrices (or, equivalently,
  non-null rank-2 antisymmetric matrices, since antisymmetric matrices
  have even rank) defined up to a non-zero scalar
factor. The correspondence is given by the mapping that assigns to the
line $l$ passing through points $\bp,\bq$ the Pl\"ucker matrix
$\mL=\mM(\bp,\bq)=\bp\bq^{\top}-\bq\bp^{\top}$.  There is an
equivalent mapping attaching to the line determined by planes
$\balpha,\bbeta$ the matrix $\mL^{*}=\mM(\balpha,\bbeta)$.  These two
matrices are related by the transformation
$*:\mL=(m_{ij})\mapsto\mL^{*}$ where
\[
\mL^{*}=\begin{pmatrix}0 & m_{34} & m_{42} & m_{23}\\
-m_{34} & 0 & m_{14} & m_{31}\\
-m_{42} & -m_{14} & 0 & m_{12}\\
-m_{23} & -m_{31} & -m_{12} & 0
\end{pmatrix}.
\]

We also recall that the intersection of the line of Pl\"ucker matrix
$\mL$ with the plane $\bpi=(u_{1},u_{2},u_{3},u_{4})^{\top}$ is
the point $\bp=\mL\bpi$, which is zero if and only if the line $\mL$
is contained in $\bpi$.

If $\mH$ is a coordinate change $\bX'=\mH\bX$, the line $\mL$ is
written in the new coordinate system as 
\begin{equation}
\mL'=\mH\mL\mH^{\top}.\label{eq:plucker_change}
\end{equation}


The degree-two Veronese mapping $\nu_{2}$
maps a point $\bx$ to the pairwise product of its coordinates.
In particular, for $n=2,3$ we define:
\[
\begin{split}\nu_{2}(x_{1},x_{2},x_{3})^{\top} & =(x_{1}^{2},x_{1}x_{2},x_{1}x_{3},x_{2}^{2},x_{2}x_{3},x_{3}^{2})^{\top},\\
\nu_{2}(x_{1},x_{2},x_{3},x_{4})^{\top} & =(x_{1}^{2},x_{1}x_{2},x_{1}x_{3},x_{1}x_{4},x_{2}^{2},\\
& \qquad x_{2}x_{3},x_{2}x_{4},x_{3}^{2},x_{3}x_{4},x_{4}^{2})^{\top}.
\end{split}
\]

Observe that a point $\bx$ in $\bP^{2}$ belongs to the conic of matrix
$\mC=(c_{ij})$ if and only if $\bx^\top\mC\,\bx=0$, which in terms of
$\nu_2$ can be written as
\[
\nu_{2}(\bx)^{\top}\bar{\mC}=0,\quad\text{where}\quad\bar{\mC}=\Bigl(c_{11},\frac{c_{12}}{2},\frac{c_{13}}{2},c_{22},\frac{c_{23}}{2},c_{33}\Bigr)^{\top}.
\]
Hence, six points $\bq_{1},\ldots,\bq_{6}$ of the plane lie on a
conic if and only if for some non-zero vector $\bar\mC$
\begin{equation}
\label{eq:conic_as_kernel}
 \left(\nu_{2}(\bq_{1})\cdots\nu_{2}(\bq_{6})\right)^\top \bar\mC = 0,
\end{equation}
which is equivalent to the singularity of the matrix, i.e., 
\[
\det\left(\nu_{2}(\bq_{1})\cdots\nu_{2}(\bq_{6})\right)=0.
\]

Similarly, $10$ points $\bq_i$ of 3-space lie on a quadric if and only if 
\[
\det\left(\nu_{2}(\bq_{1}),\ldots,\nu_{2}(\bq_{10})\right)=0.
\]

Next, using the previous results for lines and conics, 
we characterize whether the intersection
of lines with a plane is contained in points of a conic. 
\begin{result}
\label{th:determinant_characterizes}
Given six lines with Pl\"ucker matrices $\mL_i$ and
vectors $\bpi,\ba_1,\ldots,\ba_4$ in $\Complex^4$, let us consider the
polynomial
\begin{equation}
  \begin{split}
    &D(\bpi,\ba_1,\ldots,\ba_4)=\\
    &\quad\det\left(\nu_{2}(\mL_{1}\bpi),\ldots,\nu_{2}(\mL_{6}\bpi),\nu_{2}(\ba_{1}),\ldots,\nu_{2}(\ba_{4})\right).
  \end{split}
\label{eq:determinant_definition}
\end{equation}
The set of planes $\bpi$ intersecting the six lines
in points of a conic is
defined by
\begin{displaymath}
  D(\bpi,\ba_1,\ldots,\ba_4)=0 \;\;\text{ for all
  }\;\ba_1,\ldots,\ba_4 \text{ in } \Complex^4.
\end{displaymath}
\end{result}
The proof of this result is given in Appendix~\ref{sec:Proof_determinant_characterizes}. 
%

Being each column $\nu_2(\mL_i\pi)$ of degree two in $\bpi$, the
equation $D(\bpi,\ba_1,\ldots,\ba_4)=0$ is of degree $12$ in the
coordinates of the plane. Next result, proven in Appendix~\ref{sec:Proof_determinant_factorizes}, 
shows that we can factor out four trivial linear factors and obtain an 8th-degree
polynomial in $\bpi$ which is also shown to be independent of the $\ba_j$.
\begin{result} 
\label{th:determinant_factorizes}
The set of planes $\bpi$ intersecting the six lines
  with Pl\"ucker matrices $\mL_i$ in points of a conic is given by
  the 8th-degree polynomial equation
  \begin{equation}
    \label{eq:F}
    F(\bpi) = 0
  \end{equation}
defined by the relationship
\begin{equation}
\begin{split} & D(\bpi,\ba_1,\ldots,\ba_4)=\\
 &\quad\det(\ba_{1},\ldots,\ba_{4})\,(\bpi^{\top}\ba_{1})\cdots(\bpi^{\top}\ba_{4})\, F(\bpi),
\end{split}
\label{eq:determinant_factorizes}
\end{equation}
Furthermore, the polynomial $F$ does not depend on the variables $\ba_j$. 
\end{result}

The surface of $\bP^{3*}$ (set of planes of 3D space) 
given by the planes intersecting the six lines in points of 
a conic will be called the \emph{six-line conic variety} (SLCV).

\subsection{The SLCV given by the isotropic lines of three finite
 square-pixel cameras}

In the particular case of three finite square-pixel
cameras (i.e., cameras such that the optical center is
  a finite point), the
configuration of the isotropic lines, which intersect  pairwise
in the optical centers $\bC_{i}$, allows to further simplify the
factorization from
Result~\ref{th:determinant_factorizes}. 
\begin{result}
Let us consider six lines
  with Pl\"ucker matrices $\mL_i$ such that the pairs $\{\mL_1,\mL_2\}$,
  $\{\mL_3,\mL_4\}$ and $\{\mL_5,\mL_6\}$ are intersecting in points
  $\bC_1,\bC_2$ and $\bC_3$, respectively. Then the set of planes
  $\bpi$ intersecting the six lines in points of a conic and not
  passing through any of the intersections $\bC_i$ are 
  real zeros of a 5th-degree polynomial $G$ defined by:
\begin{equation}
F(\bpi)=(\bpi^{\top}\bC_{1})(\bpi^{\top}\bC_{2})(\bpi^{\top}\bC_{3})\,
G(\bpi).
\label{eq:5thdegree_2}
\end{equation} 
\end{result}
\begin{proof}
  We have proved in Result~\ref{th:determinant_factorizes} that the
  set of planes intersecting the lines $\mL_i$ in points of a conic is
  given by the zeroes of an 8th-degree polynomial, $F(\bpi)$. Since
  the lines intersect by pairs at the points $\bC_i$, there are
  trivial solutions of the equation $F(\bpi)=0$, namely those
  corresponding to planes passing through any of the points $\bC_i$.
  In fact, any plane through any of the $\bC_i$ intersects the six
  lines in at most five different points. Since five points always lie
  on a conic, all planes through any of the $\bC_i$ are zeros of
  $F$.  Therefore, we can further factorize the polynomial $F$ as in
  equation~\eqref{eq:5thdegree_2}. The planes intersecting the lines
  in points of a conic and not passing through the points $\bC_i$ are
  solutions of (\ref{eq:F}) and, since
  $\bpi^{\top}\bC_{i}\not=0$, they must be solutions of  $G(\bpi)=0$. \qed\end{proof}


\begin{corollary}Given a projective reconstruction for three
finite cameras, the planes at infinity compatible with the 
square-pixel property of the cameras are real zeros of a $5$th-degree
polynomial $G$ given by:
\begin{equation}
F(\bpi)=(\bpi^{\top}\bC_{1})(\bpi^{\top}\bC_{2})(\bpi^{\top}\bC_{3})\,
G(\bpi).
\label{eq:5thdegree}
\end{equation}\end{corollary}
\begin{proof}It is an immediate consequence of the previous result and
the fact that the isotropic lines intersect by pairs in the optical
centers of the cameras.\qed\end{proof}

In the context of camera
autocalibration, we will also use the term SLCV for the 5th-degree
variety, from which the three trivial linear factors have been removed.

A straightforward application of this result is the direct
obtainment of the candidate planes at infinity, when two
points at infinity are known, as the real solutions of a fifth-degree polynomial
equation in one variable (see~\cite{Carballeira}).

\begin{corollary} If two points
of the plane at infinity are known, there are at most five candidate
planes at infinity which can be found solving the 5th-degree equation
in the homogeneous coordinates $\lambda:\mu$ 
\begin{displaymath}
G(\lambda\balpha+\mu\bbeta)=0,
\end{displaymath}
 where $\balpha$ and $\bbeta$ are any two planes intersecting in
the line determined by the two points at infinity. \end{corollary}

\subsection{Singularity of the principal planes}

From now on, we suppose as above that the six lines are the isotropic
lines of three square-pixel cameras. The following result is the key
to find a computationally efficient way to parameterize the set of
associated candidate planes. For our configuration
of three cameras, we will say that a principal plane is \emph{generic}
if it does not contain the projection centers of any of the other
two cameras. 

We recall that a line intersects a complex projective
  hypersurface (or variety) of degree $d$ in $d$ points, counted with
  their multiplicity.  A point $\bp$ of the hypersurface $F(\bx)=0$ is
  a singular (multiple) point of multiplicity $r$ if $\lambda=0$ is an
  $r$-th order root of the equation $F(\bp +\lambda\bq)=0$, for any
  $\bq$ not in the hypersurface. Therefore, the line $\bp\bq$
  intersects the variety in at most $d-r$ additional points.  Singular
  points are useful to obtain parametrizations of the variety by
  computing its intersections with the lines through the point.
For instance, given a 5-th degree variety with a point $\bp$ of multiplicity~3, 
for each line through $\bp$ there are at most two additional points
of intersection of the line with the variety. 
The approach 
of using a triple point to parametrize a 5-th degree variety is
illustrated in Fig.~\ref{fig:IllustrationTriplePoint} for the case of a 2D curve.
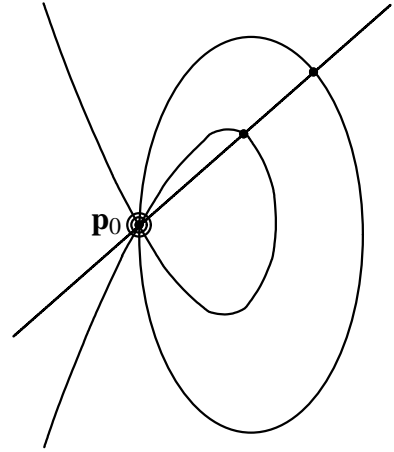
\begin{figure}[tbp]
\centering{}
\setlength{\unitlength}{0.00032489in}
\begingroup\makeatletter\ifx\SetFigFont\undefined%
\gdef\SetFigFont#1#2#3#4#5{%
  \reset@font\fontsize{#1}{#2pt}%
  \fontfamily{#3}\fontseries{#4}\fontshape{#5}%
  \selectfont}%
\fi\endgroup%
{\renewcommand{\dashlinestretch}{30}
\begin{picture}(6134,7316)(0,-10)
\put(4874,6153){\blacken\ellipse{134}{134}}
\put(4874,6153){\ellipse{134}{134}}
\put(2054,3663){\blacken\ellipse{134}{134}}
\put(2054,3663){\ellipse{134}{134}}
\thicklines
\put(3862,3498){\ellipse{3616}{6450}}
\thinlines
\put(3746,5145){\blacken\ellipse{134}{134}}
\put(3746,5145){\ellipse{134}{134}}
\thicklines
\put(2055,3662){\ellipse{382}{382}}
\put(2051,3663){\ellipse{250}{250}}
\dottedline{45}(22,1841)(6112,7240)
\path(507,7268)(507,7267)(508,7266)
	(509,7263)(510,7258)(513,7251)
	(516,7241)(520,7229)(525,7214)
	(531,7196)(538,7175)(546,7151)
	(556,7125)(566,7095)(577,7063)
	(589,7028)(602,6990)(615,6951)
	(630,6909)(645,6866)(661,6821)
	(677,6774)(694,6727)(712,6677)
	(730,6627)(748,6575)(768,6522)
	(787,6468)(808,6413)(829,6356)
	(851,6298)(874,6238)(897,6177)
	(922,6113)(947,6048)(974,5980)
	(1001,5911)(1030,5838)(1060,5764)
	(1092,5687)(1124,5608)(1158,5527)
	(1192,5445)(1227,5363)(1264,5277)
	(1300,5193)(1335,5112)(1369,5035)
	(1401,4963)(1430,4896)(1458,4835)
	(1483,4779)(1505,4728)(1525,4683)
	(1543,4643)(1559,4608)(1573,4577)
	(1585,4550)(1595,4527)(1604,4507)
	(1612,4489)(1619,4473)(1625,4459)
	(1631,4446)(1637,4434)(1642,4422)
	(1648,4409)(1654,4396)(1661,4382)
	(1669,4365)(1678,4347)(1689,4326)
	(1701,4302)(1715,4274)(1731,4243)
	(1749,4208)(1769,4168)(1792,4124)
	(1817,4076)(1845,4024)(1875,3967)
	(1907,3906)(1942,3843)(1977,3777)
	(2014,3710)(2052,3643)(2096,3567)
	(2138,3493)(2178,3425)(2216,3362)
	(2250,3304)(2281,3253)(2309,3207)
	(2333,3167)(2354,3133)(2372,3104)
	(2387,3079)(2400,3059)(2410,3042)
	(2418,3028)(2426,3016)(2432,3007)
	(2437,2999)(2442,2991)(2446,2984)
	(2451,2977)(2457,2969)(2464,2959)
	(2472,2949)(2482,2936)(2494,2921)
	(2508,2903)(2525,2882)(2545,2858)
	(2568,2830)(2594,2799)(2623,2765)
	(2655,2728)(2690,2689)(2728,2648)
	(2766,2606)(2806,2565)(2856,2515)
	(2905,2470)(2949,2430)(2989,2396)
	(3024,2366)(3054,2341)(3080,2320)
	(3100,2303)(3117,2289)(3130,2279)
	(3140,2270)(3149,2263)(3156,2258)
	(3162,2253)(3168,2250)(3174,2246)
	(3182,2242)(3192,2239)(3204,2234)
	(3220,2230)(3239,2225)(3262,2219)
	(3289,2213)(3321,2208)(3356,2203)
	(3396,2200)(3438,2198)(3481,2200)
	(3524,2206)(3564,2214)(3601,2225)
	(3635,2236)(3664,2248)(3689,2259)
	(3711,2269)(3728,2279)(3743,2288)
	(3755,2296)(3764,2303)(3772,2310)
	(3778,2316)(3784,2323)(3790,2330)
	(3796,2339)(3803,2348)(3811,2360)
	(3821,2373)(3833,2390)(3847,2410)
	(3864,2433)(3884,2461)(3907,2494)
	(3932,2531)(3959,2573)(3988,2620)
	(4016,2670)(4041,2719)(4065,2768)
	(4086,2816)(4106,2860)(4122,2901)
	(4137,2937)(4150,2970)(4161,2998)
	(4170,3022)(4177,3042)(4184,3059)
	(4189,3074)(4194,3087)(4198,3098)
	(4201,3110)(4204,3121)(4208,3133)
	(4211,3147)(4214,3164)(4218,3183)
	(4222,3206)(4226,3234)(4231,3267)
	(4236,3305)(4241,3349)(4247,3399)
	(4252,3455)(4256,3516)(4260,3581)
	(4262,3648)(4262,3712)(4261,3774)
	(4259,3833)(4256,3888)(4253,3938)
	(4250,3984)(4246,4024)(4243,4059)
	(4240,4090)(4237,4116)(4235,4139)
	(4232,4159)(4230,4176)(4228,4191)
	(4225,4205)(4223,4218)(4220,4231)
	(4217,4245)(4214,4259)(4210,4275)
	(4205,4294)(4199,4315)(4192,4339)
	(4183,4366)(4173,4398)(4161,4433)
	(4148,4472)(4132,4515)(4115,4562)
	(4095,4610)(4074,4660)(4051,4710)
	(4025,4761)(3999,4808)(3974,4852)
	(3950,4891)(3928,4925)(3908,4955)
	(3891,4981)(3876,5004)(3863,5023)
	(3852,5039)(3843,5053)(3836,5064)
	(3829,5074)(3823,5083)(3817,5091)
	(3811,5099)(3805,5106)(3797,5114)
	(3788,5121)(3777,5129)(3764,5138)
	(3748,5147)(3729,5157)(3707,5168)
	(3681,5178)(3652,5188)(3618,5198)
	(3582,5206)(3542,5212)(3501,5215)
	(3459,5214)(3418,5209)(3379,5202)
	(3344,5193)(3311,5184)(3283,5175)
	(3259,5165)(3238,5156)(3221,5148)
	(3207,5140)(3195,5132)(3185,5125)
	(3178,5118)(3171,5111)(3164,5104)
	(3158,5097)(3151,5088)(3142,5079)
	(3132,5067)(3120,5054)(3105,5039)
	(3087,5020)(3064,4998)(3038,4973)
	(3007,4943)(2972,4909)(2933,4871)
	(2890,4828)(2844,4781)(2796,4730)
	(2756,4686)(2717,4641)(2679,4598)
	(2644,4555)(2612,4516)(2582,4478)
	(2555,4444)(2531,4414)(2510,4386)
	(2492,4361)(2477,4340)(2463,4320)
	(2452,4303)(2442,4288)(2433,4275)
	(2426,4262)(2419,4251)(2412,4239)
	(2406,4228)(2399,4216)(2392,4203)
	(2384,4188)(2374,4172)(2363,4153)
	(2350,4132)(2336,4107)(2318,4079)
	(2299,4047)(2277,4011)(2252,3971)
	(2224,3926)(2194,3878)(2162,3825)
	(2128,3770)(2093,3712)(2057,3653)
	(2020,3591)(1985,3530)(1952,3473)
	(1922,3420)(1895,3372)(1871,3330)
	(1851,3293)(1834,3262)(1819,3236)
	(1808,3214)(1799,3197)(1792,3184)
	(1787,3174)(1784,3167)(1781,3161)
	(1779,3157)(1778,3153)(1777,3149)
	(1775,3144)(1773,3138)(1769,3130)
	(1765,3119)(1758,3104)(1750,3086)
	(1739,3063)(1726,3034)(1710,3000)
	(1691,2959)(1669,2912)(1645,2859)
	(1617,2800)(1587,2735)(1555,2665)
	(1522,2593)(1492,2527)(1463,2462)
	(1434,2398)(1408,2338)(1382,2281)
	(1359,2228)(1338,2179)(1319,2135)
	(1302,2096)(1287,2061)(1274,2030)
	(1262,2003)(1252,1979)(1243,1958)
	(1235,1940)(1228,1923)(1222,1908)
	(1217,1894)(1211,1881)(1206,1867)
	(1200,1853)(1194,1838)(1187,1822)
	(1180,1803)(1172,1782)(1162,1758)
	(1151,1731)(1138,1699)(1124,1664)
	(1108,1624)(1090,1580)(1070,1530)
	(1049,1477)(1026,1419)(1001,1357)
	(975,1292)(949,1226)(922,1158)
	(893,1083)(865,1011)(838,942)
	(813,876)(790,814)(768,756)
	(747,701)(728,649)(710,600)
	(694,554)(678,510)(663,468)
	(649,428)(636,390)(623,353)
	(611,319)(600,285)(589,253)
	(579,223)(570,195)(561,169)
	(553,145)(546,123)(540,103)
	(534,86)(529,72)(526,60)
	(523,51)(520,44)(519,39)
	(518,35)(517,34)(517,33)
\put(1282,3588){\makebox(0,0)[lb]{\smash{{\SetFigFont{12}{14.4}{\rmdefault}{\mddefault}{\updefault}$\bp_0$}}}}
\end{picture}
}
\caption{Illustration of the approach to obtain the points of the SLCV.\protect \\
The projective plane curve $G(\bp)=0$ is of degree five and has $\bp_{0}$ 
as a triple point. Therefore, substituting in this equation the parametric
equation of a line through this point, results in an equation of the
form $G(\lambda\bp_{0}+\mu\bq_{0})=\mu^{3}H(\lambda,\mu)$, so that
the other two points of intersection of the line and the curve are
obtained by solving the quadratic equation $H(\lambda,\mu)=0$.
}
\label{fig:IllustrationTriplePoint}
\end{figure}

Result~\ref{res:sing_planes} below provides us with three very useful
multiple planes in the SLCV (points, if we interpret the SLCV as
a surface in the space of planes, $\bP^{3*}$).

\begin{result} \label{res:sing_planes} Any generic
principal plane is a singularity of multiplicity three of the variety
of candidate planes $G(\bpi)=0$. 
\end{result} 

Observe that, by point-plane duality, the lines through 
a point $\bpi$ of $\bP^{3*}$ correspond in $\bP^{3}$ to pencils of planes, 
i.e., sets of planes through a given line (its base) or, equivalently, 
linear combinations of two given planes.
Result~\ref{res:sing_planes}, proven in Appendix~\ref{sec:Proof_sing_planes}, suggests to 
parameterize the set of
candidate planes considering the pencils of planes with base contained
in one of the principal planes. 
We will denote by $\bpi_{i}$, $i=1,2,3$ the principal planes of the cameras.
Let us assume that
$\bpi_{1}$ is a generic principal plane. To each pencil of planes
that includes $\bpi_{1}$ we will associate two candidate planes.
To parameterize the pencil, we can consider its element $\bxi$ through
$\bC_{2}$, so the pencil is given by the planes of the form $\lambda\bpi_{1}+\mu\bxi$
(see Fig.~\ref{fig:SixLinesAlgorithm}).

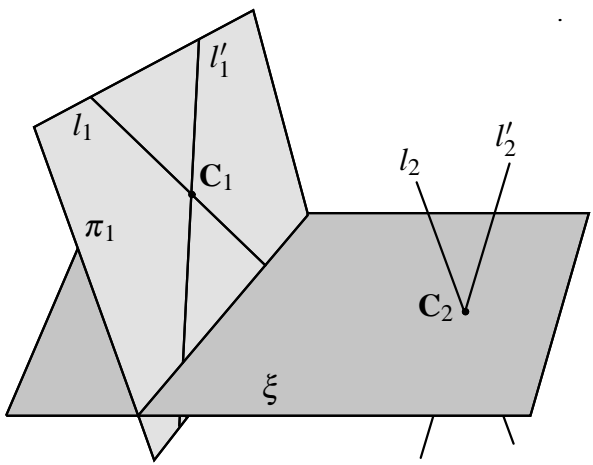
\begin{figure}
\centering 
\setlength{\unitlength}{0.00063333in}
\begingroup\makeatletter\ifx\SetFigFont\undefined%
\gdef\SetFigFont#1#2#3#4#5{%
  \reset@font\fontsize{#1}{#2pt}%
  \fontfamily{#3}\fontseries{#4}\fontshape{#5}%
  \selectfont}%
\fi\endgroup%
{\renewcommand{\dashlinestretch}{30}
\begin{picture}(4731,3705)(0,-10)
\allinethickness{1.000pt}%
\texture{8101010 10000000 444444 44000000 11101 11000000 444444 44000000 
	101010 10000000 444444 44000000 10101 1000000 444444 44000000 
	101010 10000000 444444 44000000 11101 11000000 444444 44000000 
	101010 10000000 444444 44000000 10101 1000000 444444 44000000 }
\shade\path(1004,586)(1207,12)(1476,364)(1004,586)
\path(1004,586)(1207,12)(1476,364)(1004,586)
\path(3360,31)(3716,1224)
\path(3724,1202)(4110,147)
\path(1473,1521)(1409,292)
\path(1473,1521)(1409,292)
\texture{44000000 aaaaaa aa000000 8a888a 88000000 aaaaaa aa000000 888888 
	88000000 aaaaaa aa000000 8a8a8a 8a000000 aaaaaa aa000000 888888 
	88000000 aaaaaa aa000000 8a888a 88000000 aaaaaa aa000000 888888 
	88000000 aaaaaa aa000000 8a8a8a 8a000000 aaaaaa aa000000 888888 }
\shade\path(19,375)(4246,375)(4719,2026)
	(711,2026)(12,382)(19,375)
\path(19,375)(4246,375)(4719,2026)
	(711,2026)(12,382)(19,375)
\texture{8101010 10000000 444444 44000000 11101 11000000 444444 44000000 
	101010 10000000 444444 44000000 10101 1000000 444444 44000000 
	101010 10000000 444444 44000000 11101 11000000 444444 44000000 
	101010 10000000 444444 44000000 10101 1000000 444444 44000000 }
\shade\path(2445,2006)(1076,376)(236,2726)
	(2007,3678)(2445,2027)(2445,2006)
\path(2445,2006)(1076,376)(236,2726)
	(2007,3678)(2445,2027)(2445,2006)
\path(694,2969)(2105,1600)
\path(694,2969)(2105,1600)
\path(1569,3439)(1439,808)
\path(1569,3439)(1439,808)
\drawline(4482,3598)(4482,3598)
\path(3327,2277)(3713,1222)
\put(3722,1222){\blacken\ellipse{40}{40}}
\put(3722,1222){\ellipse{40}{40}}
\put(1511,2177){\blacken\ellipse{40}{40}}
\put(1511,2177){\ellipse{40}{40}}
\path(3721,1236)(4077,2429)
\put(1655,3135){\makebox(0,0)[lb]{{\SetFigFont{12}{14.4}{\rmdefault}{\mddefault}{\updefault}$l'_1$}}}
\put(1568,2210){\makebox(0,0)[lb]{{\SetFigFont{12}{14.4}{\rmdefault}{\mddefault}{\updefault}$\bC_1$}}}
\put(2074,460){\makebox(0,0)[lb]{{\SetFigFont{12}{14.4}{\rmdefault}{\mddefault}{\updefault}$\bxi$}}}
\put(3337,1145){\makebox(0,0)[lb]{{\SetFigFont{12}{14.4}{\rmdefault}{\mddefault}{\updefault}$\bC_2$}}}
\put(3187,2335){\makebox(0,0)[lb]{{\SetFigFont{12}{14.4}{\rmdefault}{\mddefault}{\updefault}$l_2$}}}
\put(3963,2494){\makebox(0,0)[lb]{{\SetFigFont{12}{14.4}{\rmdefault}{\mddefault}{\updefault}$l'_2$}}}
\put(638,1800){\makebox(0,0)[lb]{{\SetFigFont{12}{14.4}{\rmdefault}{\mddefault}{\updefault}$\bpi_1$}}}
\put(546,2619){\makebox(0,0)[lb]{{\SetFigFont{12}{14.4}{\rmdefault}{\mddefault}{\updefault}$l_1$}}}
\end{picture}
} 
\caption{Elements in the parameterization of the SLCV given by Result~\ref{res:sing_planes}.}
\label{fig:SixLinesAlgorithm} 
\end{figure}

Result~\ref{res:sing_planes} guarantees that the quintic polynomial $G(\lambda\bpi_{1}+\mu\bxi)$ factorizes as 
\begin{equation}
G(\lambda{\bpi_{1}}+\mu\bxi)=\mu^{3}H(\lambda,\mu)\label{eq:triple_solution}
\end{equation}
 where $H$ is a homogeneous polynomial of degree $2$ whose zeroes
provide the two candidate planes associated to $\bxi$.

\begin{remark}
\label{res:non_generic_planes_3_cameras}
Observe that only one generic principal plane is
  required in order to parametrize the set of solutions in the
  proposed way. The non-existence of generic principal planes
  constitutes a highly singular case, in which each optical center
  lies in the intersection of its principal plane with one of the other two. 
  Result~\ref{res:sing_planes} can be easily extended to show that
  a non-generic principal plane containing exactly two optical centers 
  is a singularity of multiplicity two. 
  Therefore, using such a plane to parametrize the set of candidate planes $G(\bpi)=0$ 
  implies the solution of a polynomial equation of degree three for each line in the plane.
\end{remark}  

Finally, the following result can be useful as will be
  discussed in Sect.~\ref{sec:discussion}.
\begin{result}
\label{res:included_pencils}
  The pencils of planes determined by any pair of principal planes is
  contained in the variety of candidate planes.
\end{result}
\begin{proof}
  It is enough to observe that the base of such a pencil intersect the
  corresponding isotropic lines in four points and any six points,
  four of which are aligned, lie on a conic.
\end{proof}

\section{Algorithms}
\label{sec:Algorithms}

We have seen that the set of planes at infinity
compatible with three cameras with square pixels is a surface given by
a fifth-degree equation
\[
G(\bpi)=0.
\]
This will allow to compute the plane at infinity by means of a
two-dimensional search if enough additional data about the cameras or
the scene are available. In the first part of this section we will
address a convenient way to parameterize the surface of candidate
planes, while in the second part we will focus on the particular case
in which the additional information stems from the presence of two or
more auxiliary square-pixel cameras.

\subsection{Parameterization of the candidate planes}

Exploiting Result~\ref{res:sing_planes} and assuming that the first
principal plane $\bpi_{1}$ is generic, our algorithm will sweep the set
of candidate planes at infinity corresponding to three square-pixel
cameras. We will define a one-to-two mapping attaching
to each real line $l$ of $\bpi_{1}$ the two candidate planes $\bchi_1$,
$\bchi_2$ containing $l$, i.e., the intersection of the pencil of
planes through $l$ with the variety of candidate planes.
Hence, if we parameterize the set of real lines $l$ of $\bpi_1$ that do not
contain the optical center $\bC_1$, we will obtain accordingly a
two-fold parameterization of the set of candidate planes.  To this
purpose we fix a point $\br$ in $l_{1}$ which, together with the
optical center $\bC_1$ will parameterize the points $\bq$ in  $l_1$ as
$\bq=\br+z\bC_{1}$, $z$ being a complex number\footnote{The fixed
  point $\br$ may be chosen as the intersection of $l_{1}$ with the
  plane whose coordinates coincide with those of the point
  $\bC_1$.}. We now
define $l$ as the line through $\bq$ and $\bar\bq$ (see Fig.~\ref{fig:SixLinesAlgorithm_notation}). 
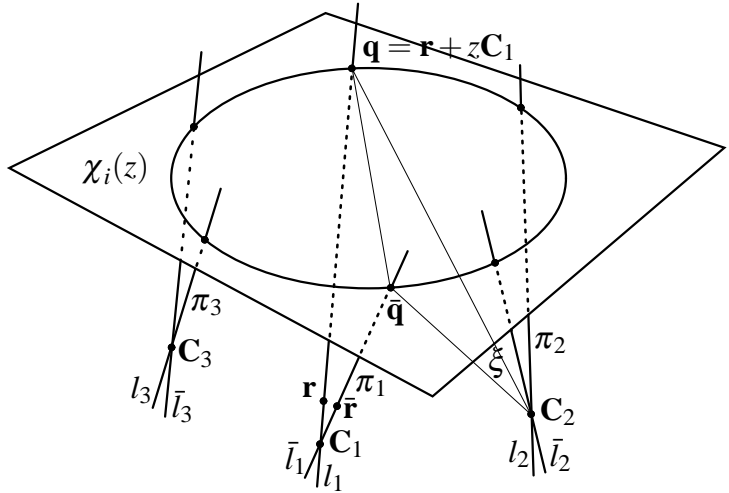
\begin{figure*}
\centering \setlength{\unitlength}{0.00047489in}
\begingroup\makeatletter\ifx\SetFigFont\undefined%
\gdef\SetFigFont#1#2#3#4#5{%
  \reset@font\fontsize{#1}{#2pt}%
  \fontfamily{#3}\fontseries{#4}\fontshape{#5}%
  \selectfont}%
\fi\endgroup%
{\renewcommand{\dashlinestretch}{30}
\begin{picture}(7896,5464)(0,-10)
\put(1813,1611){\blacken\ellipse{70}{70}}
\put(1813,1611){\ellipse{70}{70}}
\put(3437,543){\blacken\ellipse{70}{70}}
\put(3437,543){\ellipse{70}{70}}
\put(4213,2271){\blacken\ellipse{70}{70}}
\put(4213,2271){\ellipse{70}{70}}
\put(3785,4690){\blacken\ellipse{70}{70}}
\put(3785,4690){\ellipse{70}{70}}
\put(3624,964){\blacken\ellipse{70}{70}}
\put(3624,964){\ellipse{70}{70}}
\put(3474,1017){\blacken\ellipse{70}{70}}
\put(3474,1017){\ellipse{70}{70}}
\put(5747,874){\blacken\ellipse{70}{70}}
\put(5747,874){\ellipse{70}{70}}
\put(2175,2799){\blacken\ellipse{70}{70}}
\put(2175,2799){\ellipse{70}{70}}
\put(5355,2544){\blacken\ellipse{70}{70}}
\put(5355,2544){\ellipse{70}{70}}
\put(2055,4047){\blacken\ellipse{70}{70}}
\put(2055,4047){\ellipse{70}{70}}
\put(5636,4257){\blacken\ellipse{70}{70}}
\put(5636,4257){\ellipse{70}{70}}
\thicklines
\put(3969,3479){\ellipse{4320}{2430}}
\thinlines
\path(3788,4684)(5742,873)(4204,2279)(3788,4684)
\thicklines
\path(33,3590)(4668,1070)(7863,3815)
	(3498,5300)(33,3590)
\path(2138,4870)(2056,4052)
\dottedline{150}(2187,2821)(2069,2467)
\dottedline{150}(2056,4029)(1914,2585)
\path(2348,3362)(2177,2802)
\path(1912,2579)(1740,846)
\path(2073,2478)(1604,942)
\path(3860,1497)(3271,171)
\dottedline{113}(4210,2281)(3863,1506)
\path(4395,2672)(4220,2280)
\path(3532,1686)(3405,75)
\dottedline{113}(3791,4708)(3539,1690)
\path(3853,5405)(3790,4692)
\path(5210,3114)(5352,2533)
\dottedline{113}(5353,2549)(5533,1795)
\path(5529,1798)(5904,230)
\dottedline{113}(5640,4217)(5715,2012)
\path(5717,1973)(5775,196)
\path(5626,4725)(5636,4250)
\put(2014,2024){\makebox(0,0)[lb]{\smash{{\SetFigFont{12}{14.4}{\rmdefault}{\mddefault}{\updefault}$\bpi_3$}}}}
\put(5269,1346){\makebox(0,0)[lb]{\smash{{\SetFigFont{12}{14.4}{\rmdefault}{\mddefault}{\updefault}$\boldsymbol{\xi}$}}}}
\put(1362,1053){\makebox(0,0)[lb]{\smash{{\SetFigFont{12}{14.4}{\rmdefault}{\mddefault}{\updefault}$l_3$}}}}
\put(5949,304){\makebox(0,0)[lb]{\smash{{\SetFigFont{12}{14.4}{\rmdefault}{\mddefault}{\updefault}$\bar{l}_2$}}}}
\put(3244,1068){\makebox(0,0)[lb]{\smash{{\SetFigFont{12}{14.4}{\rmdefault}{\mddefault}{\updefault}$\br$}}}}
\put(3473,78){\makebox(0,0)[lb]{\smash{{\SetFigFont{12}{14.4}{\rmdefault}{\mddefault}{\updefault}$l_1$}}}}
\put(1812,840){\makebox(0,0)[lb]{\smash{{\SetFigFont{12}{14.4}{\rmdefault}{\mddefault}{\updefault}$\bar{l}_3$}}}}
\put(1875,1454){\makebox(0,0)[lb]{\smash{{\SetFigFont{12}{14.4}{\rmdefault}{\mddefault}{\updefault}$\bC_3$}}}}
\put(837,3505){\makebox(0,0)[lb]{\smash{{\SetFigFont{12}{14.4}{\rmdefault}{\mddefault}{\updefault}$\boldsymbol{\chi}_i(z)$}}}}
\put(4167,1922){\makebox(0,0)[lb]{\smash{{\SetFigFont{12}{14.4}{\rmdefault}{\mddefault}{\updefault}$\bar{\bq}$}}}}
\put(5842,805){\makebox(0,0)[lb]{\smash{{\SetFigFont{12}{14.4}{\rmdefault}{\mddefault}{\updefault}$\bC_2$}}}}
\put(5500,349){\makebox(0,0)[lb]{\smash{{\SetFigFont{12}{14.4}{\rmdefault}{\mddefault}{\updefault}$l_2$}}}}
\put(3052,271){\makebox(0,0)[lb]{\smash{{\SetFigFont{12}{14.4}{\rmdefault}{\mddefault}{\updefault}$\bar{l}_1$}}}}
\put(5774,1605){\makebox(0,0)[lb]{\smash{{\SetFigFont{12}{14.4}{\rmdefault}{\mddefault}{\updefault}$\bpi_2$}}}}
\put(3813,1128){\makebox(0,0)[lb]{\smash{{\SetFigFont{12}{14.4}{\rmdefault}{\mddefault}{\updefault}$\bpi_1$}}}}
\put(3690,831){\makebox(0,0)[lb]{\smash{{\SetFigFont{12}{14.4}{\rmdefault}{\mddefault}{\updefault}$\bar{\br}$}}}}
\put(3529,453){\makebox(0,0)[lb]{\smash{{\SetFigFont{12}{14.4}{\rmdefault}{\mddefault}{\updefault}$\bC_1$}}}}
\put(3901,4833){\makebox(0,0)[lb]{\smash{{\SetFigFont{12}{14.4}{\rmdefault}{\mddefault}{\updefault}$\bq=\br+z\bC_1$}}}}
\end{picture}
}
\caption{Geometric elements used in Algorithm~\ref{alg:Algorithm1Parametrization}. 
The three square-pixel cameras are represented by their 
optical centers $\bC_j$ and their isotropic lines $l_j,\bar{l}_j$ contained in the principal planes $\bpi_j$, $j=1,2,3$.
Planes $\bpi_1$ and $\bxi$ are generators of the pencil of planes through the line joining $\bq(z)$ and $\bar{\bq}(z)$.
Within the pencil, solution planes $\bchi_i(z)$, $i=1,2$ intersect the six isotropic lines in points of a conic.
}
\label{fig:SixLinesAlgorithm_notation} 
\end{figure*}

We will consider as generators of the pencil of planes through $l$ the principal
plane $\bpi_{1}$ and a plane $\bxi$ passing through $l$ and $\bC_{2}$,
so the planes of the pencil are of the form
\begin{displaymath}
  \bchi = \lambda \bpi_1 + \mu \bxi.
\end{displaymath}

As explained in the previous section, the solutions contained in the
pencil through $l$ will be the zeroes of the polynomial
$H(\lambda,\mu)$ defined in~\eqref{eq:triple_solution}. Since it is a
homogeneous  degree two polynomial, it has an expression of the form 
\begin{equation}
H(\lambda,\mu)=A\lambda^{2}+B\lambda\mu+C\mu^{2},\label{eq:degree_2}
\end{equation}
where $A,B,C$ are polynomials in the coefficients of the Pl\"ucker
matrices $\mL_{i},\bar{\mL}_{i}$ of the isotropic 
lines\footnote{See Appendix~\ref{sec:isotropic_lines} for explicit
  formulas of $\mL_{i},\bar{\mL}_{i}, \;i=1,2,3$.} of the cameras. 
Their size, of the order of thousands of
terms, is not suitable for algorithmic use. However, a convenient
coordinate change shortens them so that they add up to about one
hundred terms.

Specifically, denoting by $\bq$ and $\bar{\bq}$ the intersection
points of the isotropic lines $l_{1}$ and $\bar{l}_{1}$ with the
plane $\bxi$ and assuming that the points $\bC_{i}$, $\bq$ and
$\bar{\bq}$ are in general position, we consider the coordinate change
$\mH$ performing the mapping 
\begin{equation}
\begin{split}\bC_{1} & \mapsto(0,0,0,1)^{\top}=&\bv_1\\
\bC_{2} & \mapsto(0,0,1,1)^{\top}=&\bv_2\\
\bC_{3} & \mapsto(0,1,-1,1)^{\top}=&\bv_3\\
\bq & \mapsto(1,i,0,0)^{\top}=&\bv_4\\
\bar{\bq} & \mapsto(1,-i,0,0)^{\top}=&\bv_5.
\end{split}
\label{eq:mapsto}
\end{equation}

The matrix of the coordinate change can be computed as 
\begin{displaymath}
  \mH=\begin{pmatrix}
    \beta_1\bv_1 & \beta_2\bv_2 &\beta_3\bv_3 &\beta_1\bv_4 
  \end{pmatrix}
  \begin{pmatrix}
    \alpha_1\bC_1 & \alpha_2\bC_2 & \alpha_3\bC_3 & \alpha_4\bq 
  \end{pmatrix}^{-1},
\end{displaymath}
where 
\begin{displaymath}
  \begin{pmatrix}
    \alpha_1 & \alpha_2 &\alpha_3 &\alpha_4 
  \end{pmatrix}^\top = 
\begin{pmatrix}
    \bC_1 & \bC_2 & \bC_3 & \bq 
  \end{pmatrix}^{-1}\bar\bq,
\end{displaymath}
and 
\begin{displaymath}
  \begin{pmatrix}
    \beta_1 & \beta_2 &\beta_3 &\beta_4 &
  \end{pmatrix}=
\begin{pmatrix}
    \bv_1 & \bv_2 &\bv_3 &\bv_4 
  \end{pmatrix}^{-1}\bv_5.
\end{displaymath}

 In these new coordinates we obtain a polynomial 
\begin{equation}
H_{0}(\lambda,\mu)=A_{0}\lambda^{2}+B_{0}\lambda\mu+C_{0}\mu^{2},\label{eq:H_0}
\end{equation}
 where the coefficients $A_{0},B_{0},C_{0}$ depend only on
 $\mL_{2},\mL_{3}$, since $\mL_1$ is constant in the new
 coordinate system.

In Algorithm~\ref{alg:Algorithm1Parametrization}, we assume that the expressions of $A_{0}$, $B_{0}$
and $C_{0}$ have been precomputed 
(see Appendix~\ref{sec:H0coeffs} for further details).

\begin{algorithm}
\normalsize
\underline{Objective}\newline
Given:

\textbf{1.} a projective calibration of three cameras with
optical centers $\bC_{i}$ and isotropic lines $l_{i},\bar{l}_{i}$
with Pl\"ucker matrices $\mL_{i},\bar{\mL}_{i}$, with generic principal
plane $\bpi_{1}$,

\textbf{2.} a fixed point $\br\not=\bC_{1}$ on the isotropic
line $l_{1}$ and

\textbf{3.} a complex number $z$,

compute the two planes through the points $\bq=\br+z\bC_{1}$
in $l_{1}$
and $\bar{\bq}=\bar{\br}+\bar{z}\bC_{1}$ in $\bar{l}_{1}$ that intersect
isotropic lines $l_{2},\bar{l}_{2},l_{3},\bar{l}_{3}$ in points that,
together with $\bq$ and $\bar{\bq}$, lie on a conic. \tabularnewline
\underline{Algorithm}
\begin{enumerate}
\item With $\bq=\br+z\bC_{1}$ compute the matrix of the coordinate change~\eqref{eq:mapsto}
and apply~\eqref{eq:plucker_change} to transform $\mL_{2}$ and
$\mL_{3}$. 
\item Using precomputed expressions, calculate the coefficients of the
second-degree homogeneous polynomial~\eqref{eq:H_0} and its roots
$(\lambda_{i},\mu_{i})$, $i=1,2$. 
\item Obtain the solution planes in the transformed coordinates as $\bchi_{i}=\lambda_{i}\bpi_{1}+\mu_{i}\bxi$,
with $\bpi_{1}=(0,0,1,0)^{\top}$, $\bxi=(0,0,1,-1)^{\top}$, $i=1,2$
and revert coordinate change for the planes $\bchi_{i}$, thus obtaining
solution planes $\bchi_{1}(z)$, $\bchi_{2}(z)$.
\end{enumerate}
\caption{Compute two planes $\bchi_{1}(z)$, $\bchi_{2}(z)$ 
intersecting the isotropic lines of three square-pixel cameras in points of a conic, where $z$ parameterizes the set of lines in $\bpi_1$.}
\label{alg:Algorithm1Parametrization}
\end{algorithm}

\subsection{Euclidean calibration with five or more cameras}

The parameterization provided by the previous Algorithm~\ref{alg:Algorithm1Parametrization} can be employed,
in particular, to perform a two-dimensional search for the plane at
infinity using the knowledge provided by two or more additional square-pixel
cameras. The next algorithm explores the set of solutions associated
to the first three cameras, aiming at the minimization of a cost function
\begin{equation}
C(z)=\min\{C_{0}(\bchi_{1}(z)),C_{0}(\bchi_{2}(z))\},\label{eq:cost-function}
\end{equation}
 where $C_{0}(\bchi)$ measures the compatibility of candidate plane
$\bchi$ with the square-pixel condition of the additional cameras. 

For a given $\bchi$, the IACs of each additional cameras,
$\iac_{i}(\bchi)$, is calculated as the conic through the projections
of the points of intersection of $l_{i}$ and $\bar{l}_{i}$, $i=1,2,3$,
with $\bchi$.  We recall that the intersection of a line of
  Pl\"ucker matrix $\mL$ with a plane $\bpi$ is given just by the
  vector $\mL\bpi$. After projecting onto the images of each
  camera these intersection points, the IAC can be computed solving
  system (\ref{eq:conic_as_kernel}), which can be done computing the
  singular vector corresponding to the least singular value of the
  matrix of the system. Then the cost $C_{0}(\bchi)$ is computed
from these IACs. In the course of the algorithm, complex solutions may
arise, albeit the actual IAC must be real. Therefore, some additional
constraints are taken into account in the design of $C_{0}(\bchi)$. In
particular, the cost $C_{0}(\bchi)$ is the maximum of the weighted sum
of four non-negative terms:
\begin{equation}
C_{0}(\bchi)=\max_{i=1..N_{c}}\sum_{k=1}^{4}\gamma_{k}C_{k}(\iac_{i}(\bchi)),\label{eq:C0def}
\end{equation}
where the weights $\gamma_{k}\ge0$, $C_{1}(\iac)$ penalizes
complex solutions, $C_{2}(\iac)$ discourages non positive-definite
IACs, $C_{3}(\iac)$ measures the square-pixel condition~(\ref{eq:square-pixel-IAC})
and $C_{4}(\iac)$ penalizes principal points outside the image domain. 

Before computing the individual costs, the IACs $\iac$ undergo two
normalization steps. First, the homogeneous matrix $\iac$ is scaled
by the unit complex number $s$ that maximizes the Frobenius norm
(given by the sum of the squares of the coefficients of the matrix)
of $\Re\{s\text{\ensuremath{\omega}}\}$. This is a constrained optimization
problem whose solution is given by a biquadratic equation in $\Re s$.
Then, $\iac$ is scaled to unit Frobenius norm, $\|\cdot\|_{F}$. 

Next, let us describe each term in~\eqref{eq:C0def}. If $\bu=\ovec(\Re\iac)$
and $\bv=\ovec(\Im\iac)$ are the vectorizations of the upper-triangular
part of the real and imaginary parts of $\iac$, 
\[
C_{1}(\text{\ensuremath{\iac}})=\|\bu\bv^{\top}-\bv\bu^{\top}\|_{F}/(\|\bu\|_{2}^{2}+\|\bv\|_{2}^{2}).
\]

The term $C_{2}$ is motivated by Sylvester's criterion, which states
that a hermitian matrix is positive-definite if and only if all of
its leading principal minors are positive. Let $D_{i}(\mA)$ $i=1,2,3$
be the three leading principal minors of matrix $\mA$ and $g(\mA)=-\sum_{i=1}^{3}\min\{0,D_{i}(\text{\ensuremath{\mA}})\}$,
then 
\[
C_{2}(\iac)=\min\{g(\Re\{\iac\}),g(\Re\{-\iac\})\},
\]
 penalizes the non positive-definiteness of the homogeneous matrix
$\iac$.

The term $C_{3}$ measures the deviation from the square-pixel condition.
We choose:
\[
C_{3}(\iac)=|\tau^2-1|+\cos^2\theta,
\]
 where $\tau=m_{y}/m_{x}=\iac_{11}/\iac_{22}$ and $\cos^2\theta=\iac_{12}^{2}/(\iac_{11}\iac_{22})$.
Observe that $C_{3}(\iac_{i})=0$ for $i=1,2,3$, by construction
of the SLCV, but not necessarily for the additional cameras. 

The term $C_{4}(\iac)$ is the taxicab distance ($\|\cdot\|_{1}$)
from the principal point $(u_{0},v_{0})$ to the boundary of the image
domain if the principal point lies outside the image domain and zero
otherwise. Formulas for $u_{0},v_{0}$ in terms of $\iac$ are well
known in the literature: $u_{0}=\diac_{13}/\diac_{33}$ and $v_{0}=\diac_{23}/\diac_{33}$,
where $\diac\sim\iac^{-1}$ is the adjoint matrix of $\iac$. This term accounts for spurious solutions
with unrealistic location of the principal points of the cameras.

\begin{algorithm}
\normalsize
\underline{Objective}

Given a projective calibration of $N_{c}\geq5$ cameras with square
pixels, projection matrices $\mP_{i}$ and isotropic lines $l_{i},\bar{l}_{i}$,
obtain a Euclidean upgrading.

\underline{Algorithm}
\begin{enumerate}
\item For each $z$ in the set
\begin{align*}
Z = & \{0\}\cup\left\{ \frac{j}{N}e^{i2\pi k/M},j=1,\ldots,N,\, k=1,\ldots,M\right\}\\
& \cup\left\{ \frac{N}{j}e^{-i2\pi k/M},j=1,\ldots,N-1,\, k=1,\ldots,M\right\},
\end{align*}
compute cost $C(z)$ as follows. 

\begin{enumerate}
\item Compute planes $\bchi_{1}(z)$ and $\bchi_{2}(z)$ using Algorithm~\ref{alg:Algorithm1Parametrization}. 
\item For each plane $\bchi$ in $\{\bchi_{1}(z),\bchi_{2}(z)\}$, 

\begin{enumerate}
\item Compute the points of intersection of isotropic lines $l_{i}$, $\bar{l}_{i}$,
$i=1,2,3$ with the plane, project them with projection matrix $\mP_{1}$
and obtain the matrix $\iac_{1}$ of the conic that they define. This
is the IAC of the first camera. 
\item Obtain the IACs $\iac_{i}$ for cameras $i=2,\ldots,N_{c}$ by transferring
$\iac_{1}$ onto them using the plane $\bchi$, $\mP_{1}$ and $\mP_{i}$
(Note that the same result would be obtained if the second or third
camera was employed). Normalize them to (unit Frobenius norm and maximum
real part). 
\item Compute $C_{0}(\bchi)$, according to equation \eqref{eq:C0def}.
\end{enumerate}
\item Compute $C(z)=\min\{C_{0}(\bchi_{1}(z)),C_{0}(\bchi_{2}(z))\}$. 
\end{enumerate}

Take $z_{0}=\arg\min_{z\in Z}C(z)$. 

\item Perform non-linear optimization to obtain a local minimum $z_{1}$
near the value $z_{0}$, repeating steps (a), (b) and (c) above for
each evaluation of the function. Choose, of the two planes attached
to $z_{1}$, the one with minimum cost $C_{0}$. The Euclidean upgrading
is determined by taking this plane as the plane at infinity and the
associated conic as the absolute conic. \end{enumerate}
\caption{Compute a Euclidean upgrading from a projective calibration.}
\label{alg:Algorithm2Upgrading}
\end{algorithm}

Algorithm~\ref{alg:Algorithm2Upgrading} performs a nonlinear optimization of cost
function~(\ref{eq:cost-function}) in two steps. In the first one,
the complex plane is sampled and the sample of minimum cost is
selected.
A sampling of the complex plane that has been empirically found to
be useful consists in splitting the plane into the unit disk and its
complement, and employing a uniform sampling in modulus and phase
for the first, and the inverses of these values for the second. 

In the second optimization step, a search for a local minimum is performed using
as starting point the complex value provided by the first step. The
Nelder-Mead method (downhill simplex method) outperforms other optimization
algorithms with numerical first derivatives such as steepest descent
or conjugate gradient.

The Euclidean upgrading may be computed from the stratified approach
in Algorithm 10.1 of \cite{Hartley-Zisserman}, with the plane at
infinity and the IAC of one of the cameras provided by Algorithm~\ref{alg:Algorithm2Upgrading}.

\section{Experiments}
\label{sec:Experiments}

\subsection{LED bar dataset}
\label{sec:LEDbar}
The proposed technique has been tested in the calibration of a set
of five synchronized square-pixel video cameras with a resolution
of $1280\times960$ pixels using a bright point device. Instead of
a single-point device, a rigid bar with three light-emitting diodes
(LEDs) is employed in the tests. In addition to providing ground truth
(the bar length constancy) to test the quality of the results, this
allows to use for comparison a calibration algorithm based on the
geometry of the captured set of 3D points.

In the experiments, a projective reconstruction of the scene is first
obtained for the five cameras. This is accomplished using Algorithm
10.1 of~\cite{Hartley-Zisserman} (eight-point algorithm) to compute the fundamental
matrix of a pair of cameras and then alternating 12.2 (triangulation)
and 7.1 (resection) of~\cite{Hartley-Zisserman}. The result is then optimized
by projective bundle adjustment~\cite[p.~434]{Hartley-Zisserman}.

Euclidean upgrading by the SLCV is compared to different methods: 
firstly, against two non-iterative techniques and, secondly, against other search-based methods.

Algorithm~\ref{alg:Algorithm2Upgrading} is compared against 
an algorithm based on the dual absolute quadric (DAQ) 
and an algorithm based on the LED bar geometry.
The Euclidean upgrading algorithm based on the DAQ assumes square
pixels and principal point at the center of the image. This provides
four linear constraints on the DAQ for each camera~\cite{Triggs97}.

The algorithm using the LED bars is based on the fact that three aligned
equidistant points determine the point at infinity of the common line~\cite[p.~50]{Hartley-Zisserman}.
Thus each captured position of the LED bar provides a point at infinity
in the projective reconstruction, and, therefore, three positions
of the bar in the projective reconstruction determine the plane at
infinity and, consequently, an affine calibration. Euclidean calibration
is then obtained by computing the affinity that makes the lengths
of the segments corresponding to captured positions of the rigid bar
as close as possible to the true bar length. This requires at least
six captures of the bar and the solution of a least-squares problem
followed by a Cholesky factorization~\cite{tresadern} to determine
the affinity up to a Euclidean motion. Affine and Euclidean calibration
are repeated 500 times using different random sets of three bars for
the affine calibration and all the bars for the Euclidean calibration,
and the reconstruction with smallest bar length variance is selected.

After the Euclidean calibration, a Euclidean bundle adjustment
is performed,~ including the enforcement of the square-pixel shape
and taking into account the lens distortion coefficients (according
to the four-parameter OpenCV model~\cite{opencv_library}).

\begin{figure}[tbp]
\begin{centering}
\includegraphics[width=0.9\columnwidth]{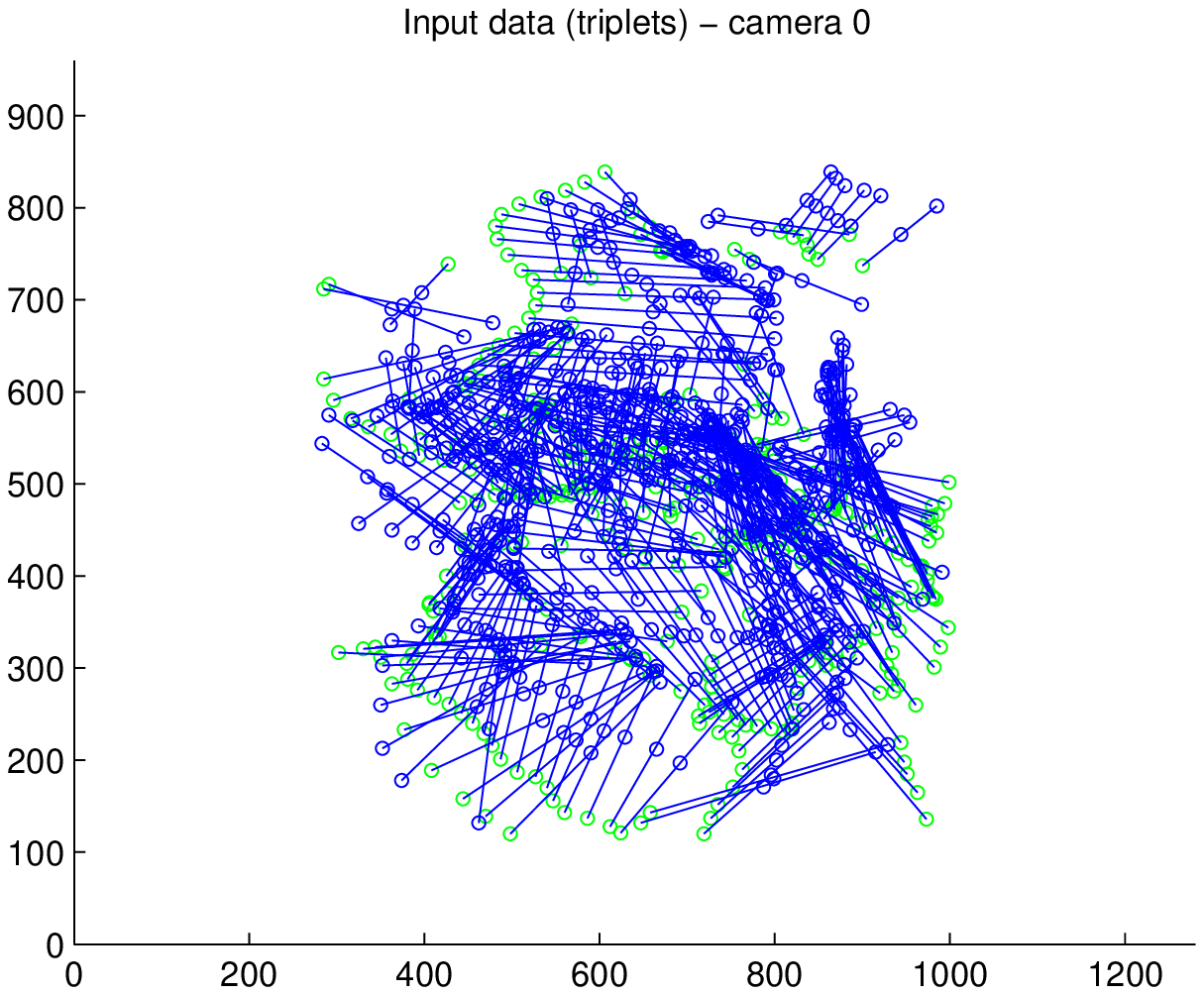}\\
\includegraphics[width=0.9\columnwidth]{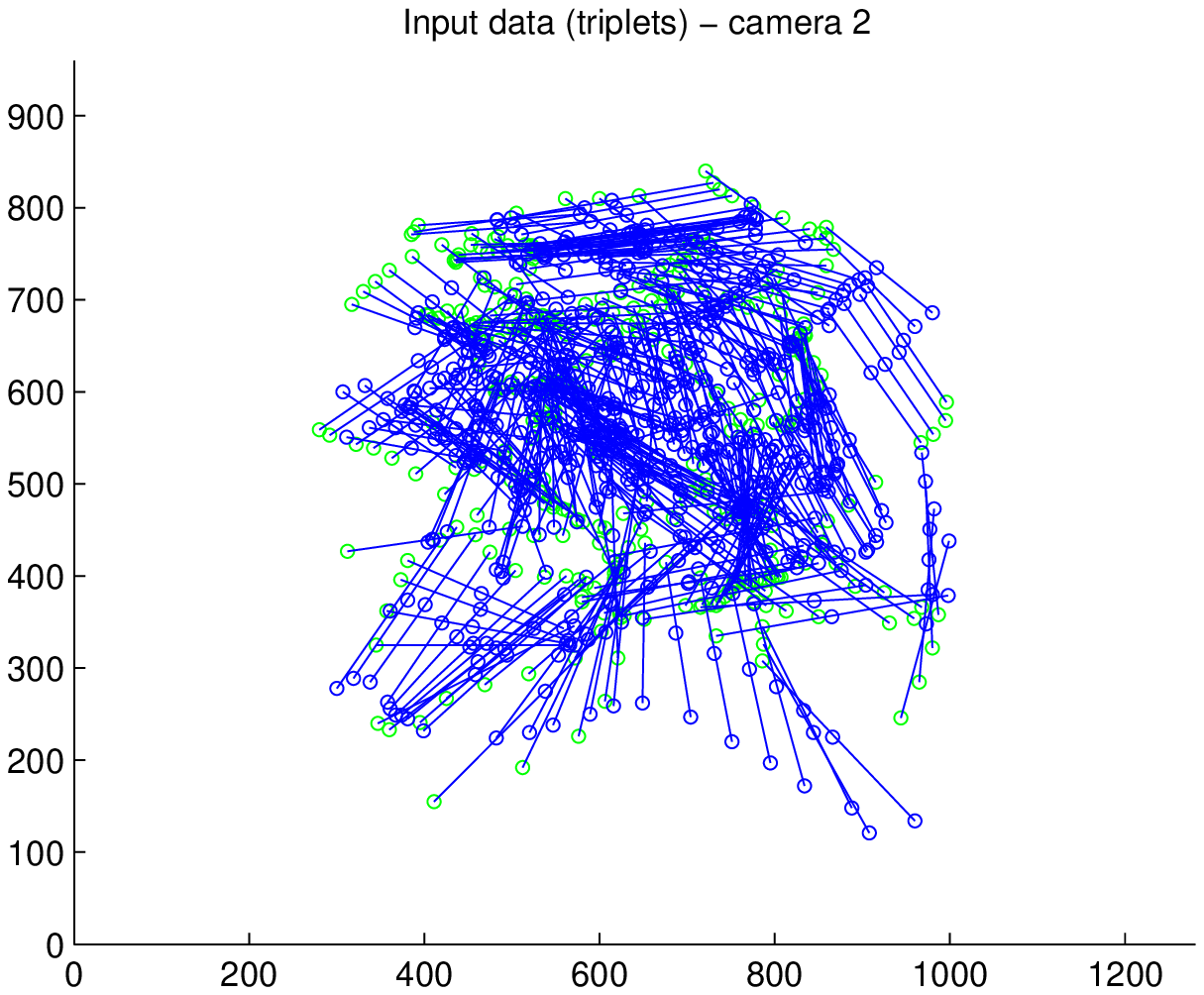}
\caption{Example of input data for the tests (triplets of aligned LEDs
  in a rigid bar). Only the points are used in the calibration with
  the proposed algorithms. The triplet structure is used for
  evaluation purposes.\label{fig:input-data} }
\end{centering}
\end{figure}
The sampling parameters in Algorithm~\ref{alg:Algorithm2Upgrading}
have been $M=N=50$, so that the cost function has been evaluated on
$2\times50\times50=5000$~points.  Figure~\ref{fig:input-data} shows a
sample of the input data for the calibration
process. Figure~\ref{fig:slcv-three-views} shows two views of the
sampled points of the SLCV computed in
Algorithm~\ref{alg:Algorithm2Upgrading}. Figure~\ref{fig:cost-function}
displays the cost function~(\ref{eq:cost-function}) for this
experiment, normalized and pseudo-colored from blue (0) to red (1). A
logarithmic transformation has been applied to enhance the
visualization of small costs. The reconstructed scene obtained with
Algorithm~\ref{alg:Algorithm2Upgrading} is shown in
Fig.~\ref{fig:rec3D}.  Table~\ref{tab:experiment1} shows the
reprojection error and the quotient between the standard deviation and
the average of the segment lengths for each of the Euclidean upgrading
techniques, before and after bundle adjustment.

\begin{figure}[tbp]
\begin{centering}
\includegraphics[width=0.9\columnwidth]{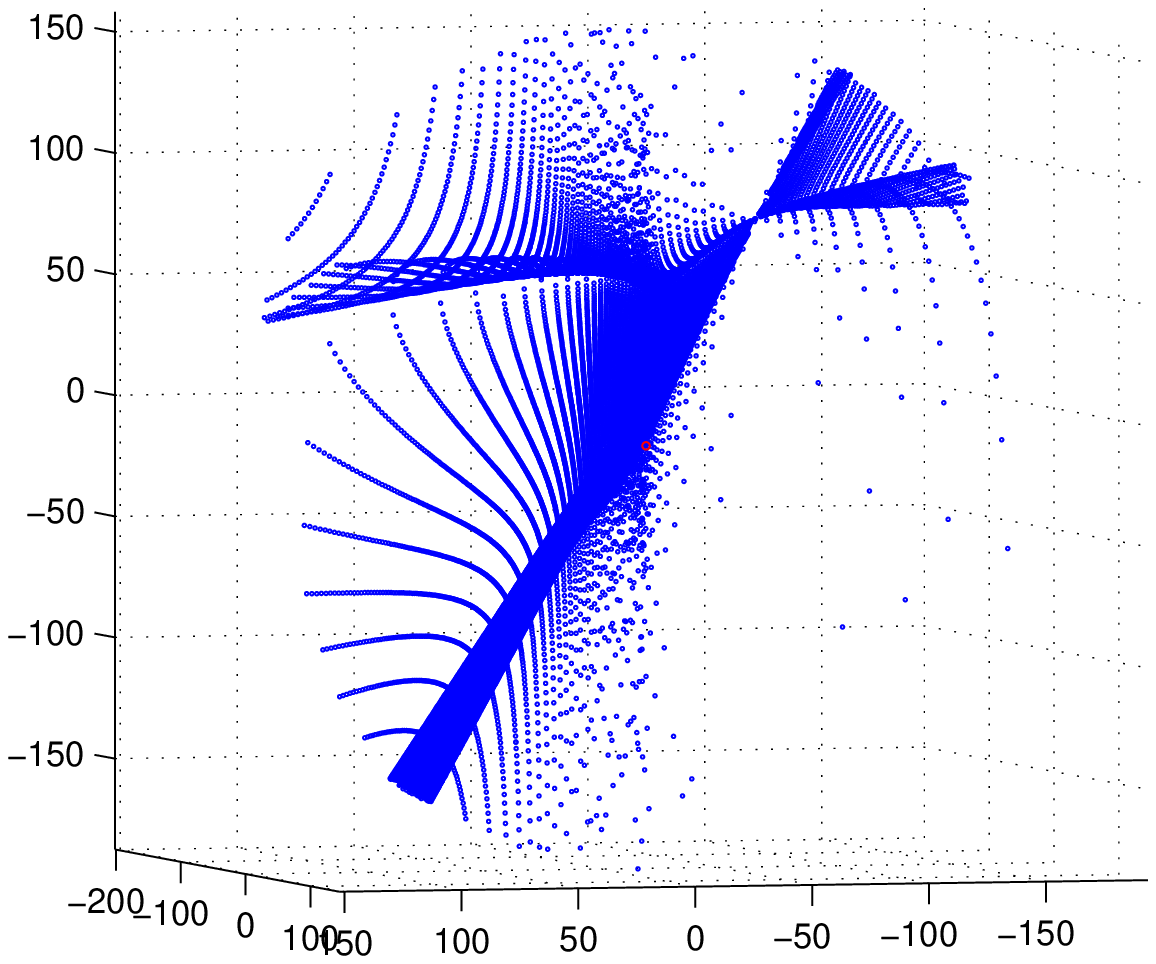}\\
\includegraphics[width=0.9\columnwidth]{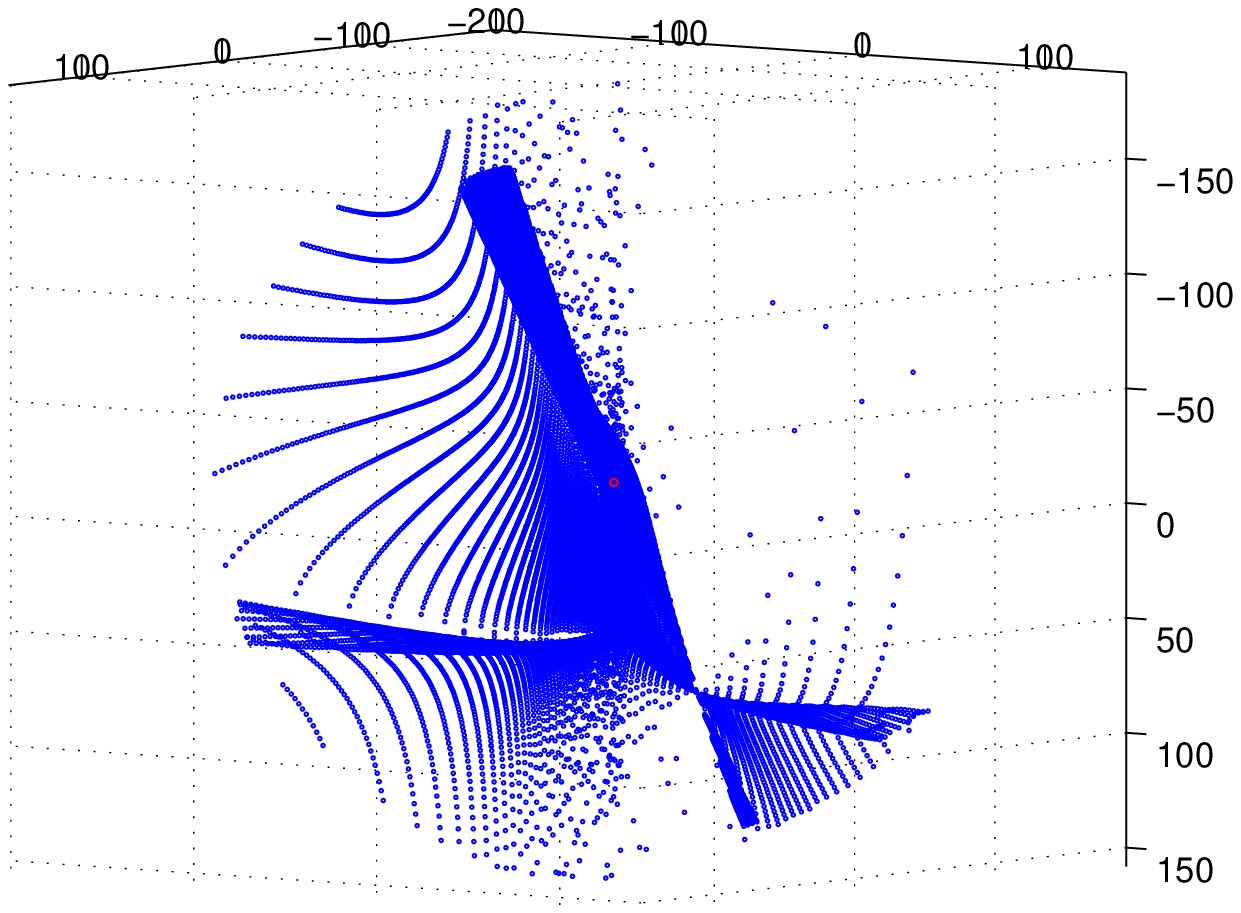} 
\caption{Two views of the sampled points of the SLCV.\label{fig:slcv-three-views} 
}\end{centering}
\end{figure}
\begin{figure}[tbp]
\begin{centering}
\includegraphics[width=0.49\columnwidth]{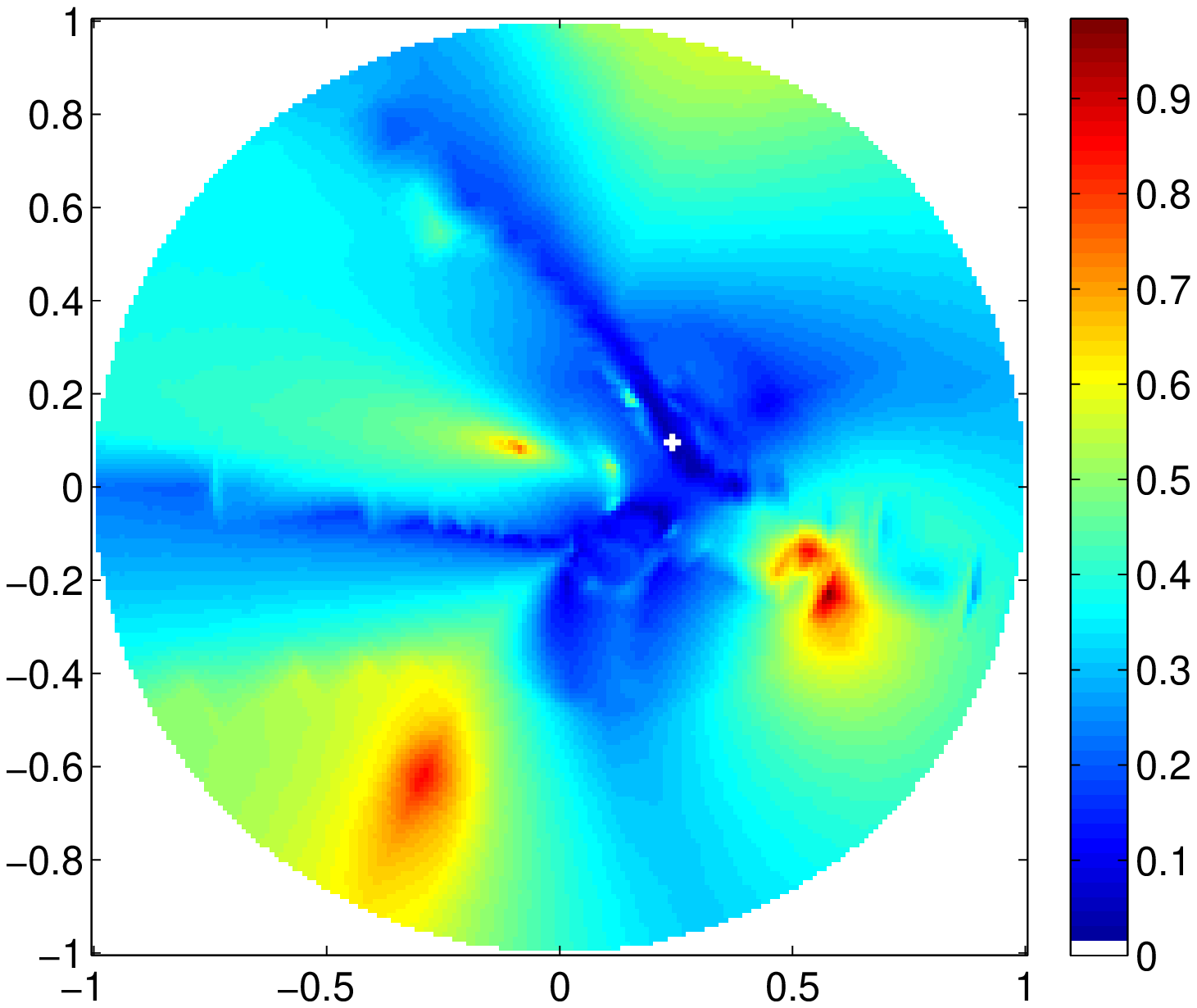}
\includegraphics[width=0.49\columnwidth]{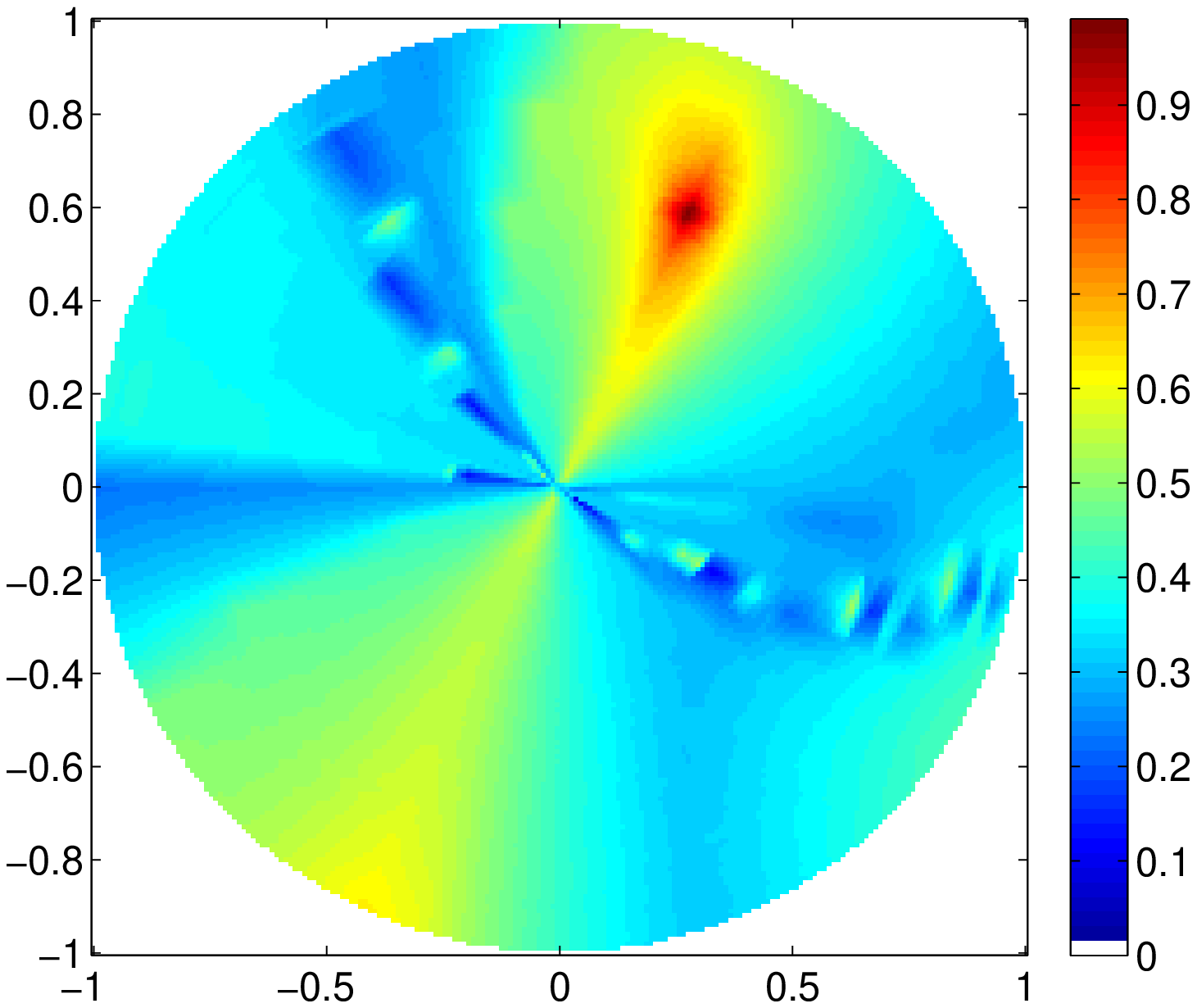}
\caption{Experiment with the LED bar dataset: plots of the sampled cost function.
Left: values at the unit disc of the complex plane, $|z|\leq1$. Right:
values at the complement of the unit disc, at positions $1/\bar{z}$.
The white cross in the left plot marks the location of the minimum.\label{fig:cost-function} }
\end{centering}
\end{figure}
\begin{figure}[htbp]
\begin{centering}
\includegraphics[width=0.9\columnwidth]{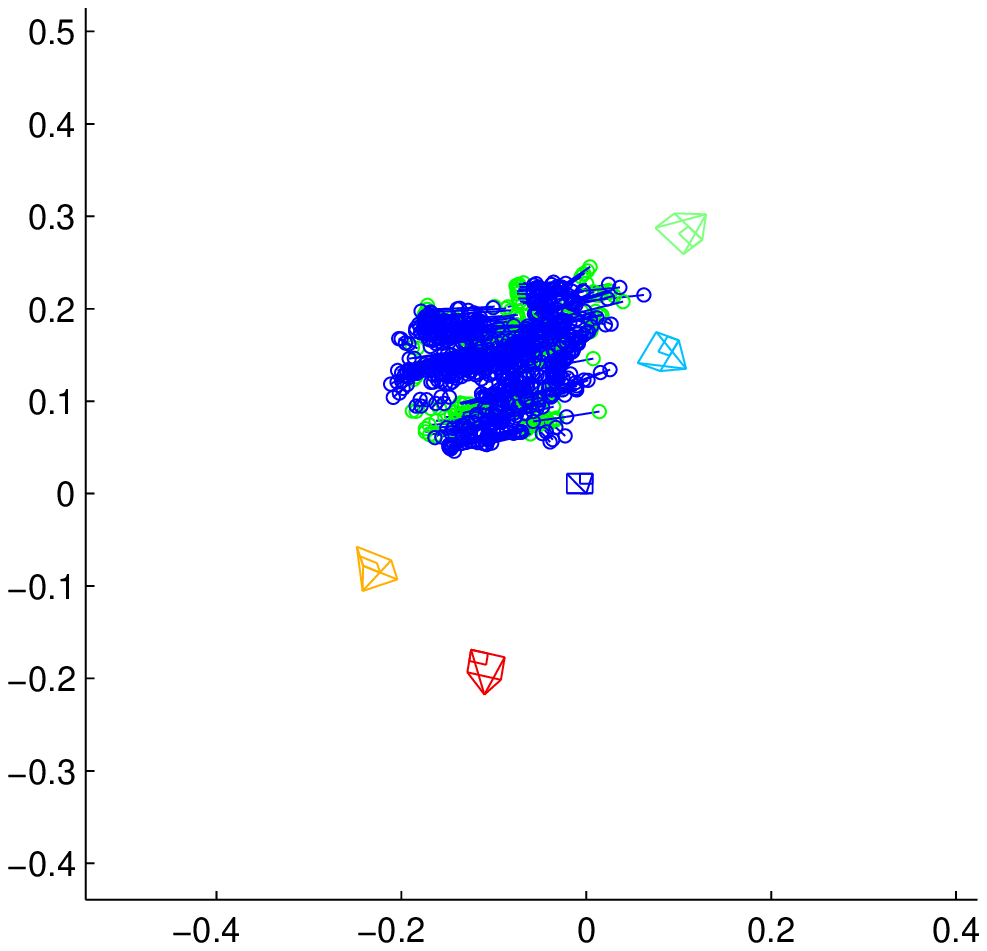}
\includegraphics[width=0.9\columnwidth]{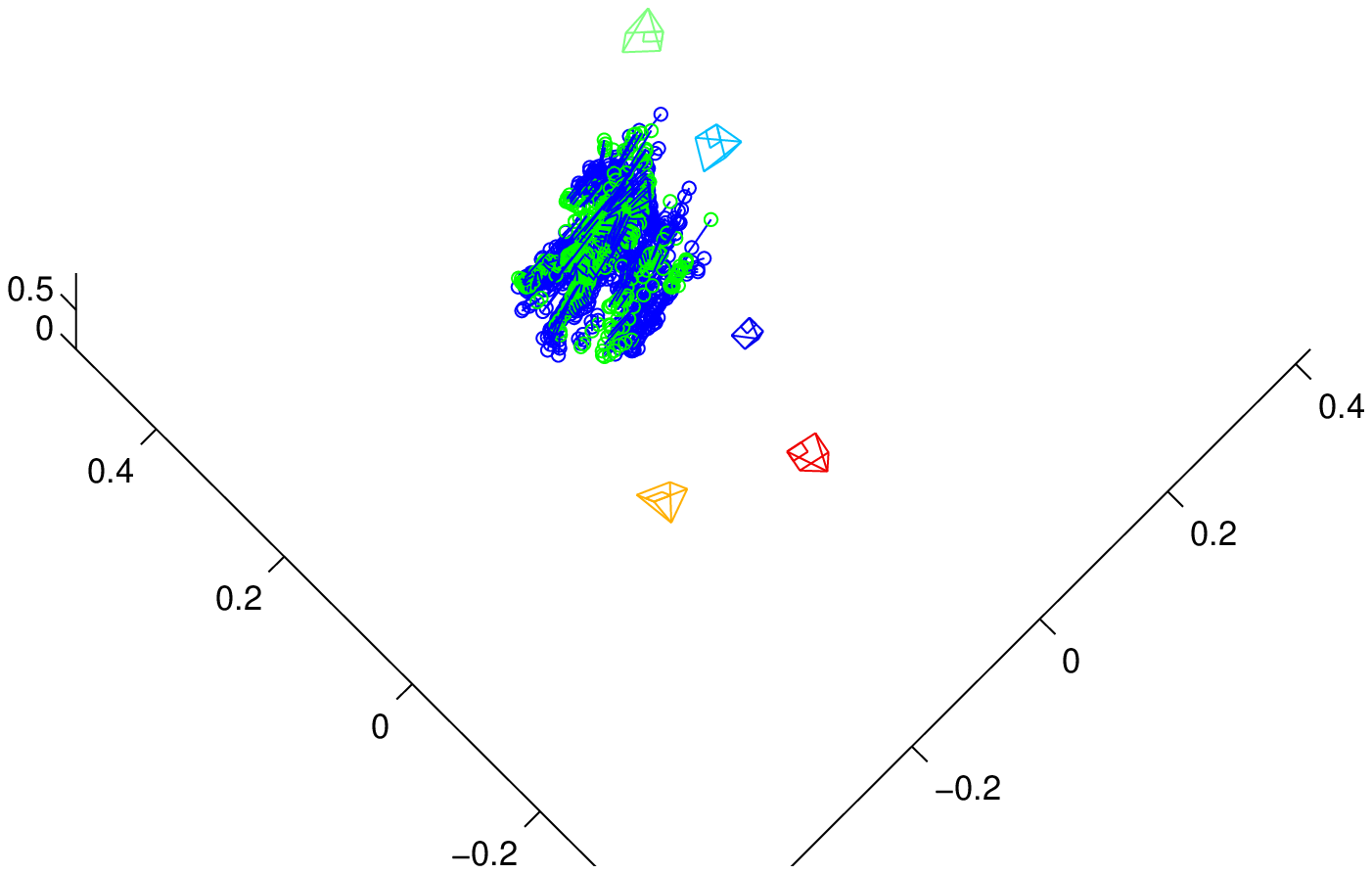}
\caption{Two views of the 3D reconstruction in one of the experiments.\label{fig:rec3D} }
\end{centering}
\end{figure}

\begin{table*}[htbp]
\caption{Results of experiment comparing calibration techniques. The common
previous projective calibration shows a reprojection error of $1.23$
pixels. $\sigma$ and $\mu$ are, respectively, the typical deviation
and the average of LED bars lengths. SLCV sampling parameters: $N = M = 50$.\label{tab:experiment1}}
\begin{centering}
\begin{tabular}{lcccccccc}
\toprule
\multirow{2}{*}{} & \multicolumn{2}{c}{DAQ} &\phantom{a} & \multicolumn{2}{c}{\textbf{SLCV}} &\phantom{a} & \multicolumn{2}{c}{Triplets}\\
\cmidrule{2-3} \cmidrule{5-6} \cmidrule{8-9}
 & Euc. calib.  & Euc. BA  && Euc. calib.  & Euc. BA  && Euc. calib.  & Euc. BA \\
\midrule
Rep. error  & $15.04$  & $1.74$  && $1.71$  & $0.64$  && $13.56$  & $0.63$ \tabularnewline
$\sigma/\mu$  & 0.15  & 0.029  && 0.049  & 0.0098  && 0.026  & 0.0048 \tabularnewline
\bottomrule
\end{tabular}
\par\end{centering}
\end{table*}

A second set of comparisons has been performed with
two algorithms based on 3D search in the set of planes of
space. The first of these algorithms is the one given in~\cite{Hartley99}
tailored to the case of square-pixel cameras and varying parameters. 
The second one makes use of a cost function based
on the transfer of cyclic points using the candidate 
plane at infinity and is described in Algorithm~\ref{alg:TransferFiveCircPoints}.

\begin{algorithm}
\normalsize
\underline{Objective}

Given the projective reconstruction of $N_{c}\geq 5$ cameras $\{\mP^{k}\}_{k=1}^{N_{c}}$ 
and the coordinates $\bpi$ of a candidate plane at infinity, compute a cost of the fitting of the plane to the square-pixel constraints.

\underline{Algorithm}

\begin{enumerate}
\item Transfer five selected imaged cyclic points to all images.
These points are the projections of the points where $\bpi$ meets the isotropic lines (Appendix~\ref{sec:isotropic_lines}) $\mL_{j},j=1,\ldots,5$ of some selected cameras, 
i.e., $\bx^{k}_j=\mP^{k}\mL_j\bpi$.

\item Compute the candidate IAC $\iac_{k}$ of each camera by fitting a conic to the set of five transferred points. 
This requires solving a system like~\eqref{eq:conic_as_kernel}, $\bx^{k\top}_j\iac_{k}\bx^{k}_j=0$, whose solution $\iac_{k}$ is the singular vector corresponding to the least singular value of the matrix of the system. 

\item Measure the how far the obtained candidate IACs are from the square-pixel
hypothesis. The cost is, based on~\eqref{eq:square-pixel-IAC}, $c(\pi)=\sum_{k=1}^{N_c}\left(|\mathbf{I}^{\top}\iac_{k}\mathbf{I}| +|\bar{\mathbf{I}}^{\top}\iac_{k}\bar{\mathbf{I}}|\right)$, where $\mathbf{I}=(1,i,0)^{\top}$ and $\bar{\mathbf{I}}=(1,-i,0)^{\top}$.
\end{enumerate}

\caption{\;Compute the cost of a candidate plane at infinity 
by transferring five cyclic points to all image planes and estimating 
the corresponding IACs of the cameras.}
\label{alg:TransferFiveCircPoints}
\end{algorithm}

The compared methods search for the plane at infinity, but while~\cite{Hartley99}
and Algorithm~\ref{alg:TransferFiveCircPoints} use a direct parameterization of
$\bpinf$ by its three coordinates in a quasi-affine reconstruction,
the SLCV method reduces the dimension of the search space from three
to two by searching over the surface (i.e., two-dimensional variety)
of candidate planes at infinity in a general projective reconstruction. 
%

It must first be pointed out that it has not been possible to obtain
valid results with any of the 3D-search algorithms without previous obtainment
of a quasi-affine reconstruction~\cite{Hartley-Zisserman,Hartley98},
an unnecessary step in the case of the presented SLCV-based
algorithm. Table~\ref{tab:Results-of-search-methods} compares the
performance of the SLCV algorithm (without quasi-affine upgrading)
with those of the two 3D search algorithms (with quasi-affine
upgrading), using the LED bar data set. The three algorithms provide
similar results for similar total number of points.

As for the computational cost of the three compared algorithms, 
the cost in the case of each iteration of the SLCV algorithm 
is the sum of the cost of the plane generation and the plane evaluation, 
both of which are comparable to the cost of plane evaluation in the two other algorithms.

\begin{table}
\caption{\label{tab:Results-of-search-methods}Experiment with five cameras 
viewing a moving LED bar. Results of search-based autocalibration
algorithms with \emph{a posteriori} square-pixel enforcement.
$\sigma$ and $\mu$ are, respectively, the typical deviation and the average
of LED bars lengths. $N$, $M$ have been chosen so that the number of samples is $2NM \approx N_{s}$.}
\begin{centering}
\begin{tabular}{lccc}
\toprule
\# samples & $N_{s}=10^{3}$ & $N_{s}=10^{3}$ & $N=M=22$\tabularnewline
Rep. error & 10.34 & 2.57 & 1.97\tabularnewline
$\sigma/\mu$ & 0.051 & 0.050 & 0.055\tabularnewline
\midrule
\# samples & $N_{s}=20^{3}$ & $N_{s}=20^{3}$ & $N=M=63$\tabularnewline
Rep. error & 1.63 & 1.21 & 1.06\tabularnewline
$\sigma/\mu$ & 0.032 & 0.032 & 0.035\tabularnewline
\bottomrule
\end{tabular}
\par\end{centering}
\end{table}

\subsection{Checkerboard dataset}
\label{sec:CheckerboardScene}

The developed method has also been tested on a different set of 5
to 10 images that contain checkerboard patterns~\cite{ronda08}. 
This calibration rig provides an additional validation of the results.
The images, of size $1280\times960$ pixels, were acquired with a Sony
 DSC-F828 digital camera. To test varying parameters, the equivalent
focal length (in a 35 mm film) of the camera was set to 50 mm in eight
of the images and to 100 mm in the remaining two. Variations due to
auto-focus were not controlled. For this range of focal lengths, the
lens distortion (radial and tangential) can be neglected, so it is
assumed to be zero.

For each of the following experiments, scale-invariant key points (SIFT~\cite{Lowe99}) 
are detected and matched across images, obtained using~\cite{SnavelySS08}.
Then, a projective reconstruction of the scene is obtained, as already explained.
Normalization of coordinates (``preconditioning'') is applied, as it
is essential
to improve the numerical conditioning of the equations in the different
estimation problems involved (fundamental matrix, triangulation, resection
and bundle adjustment).  Table~\ref{tab:CheckerboardProjective}
summarizes the parameters of the projective reconstructions. The resulting
set of projection matrices are the input to the autocalibration algorithms.

\begin{table}[h]
\caption{Projective reconstructions.}
\label{tab:CheckerboardProjective}
\begin{centering}
\begin{tabular}{lcccc}
\toprule
Scene & \multicolumn{2}{c}{Checkerboard} & \multicolumn{2}{c}{Plaza de la Villa}\\
\midrule
\# images & 10 & 5 & 16 & 5\tabularnewline
\# \mbox{3-D} points & 1494 & 1199 & 8555 & 1083\tabularnewline
\# image points & 6660 & 4472 & 35181 & 3483\tabularnewline
Rep. error (BA) & 0.22268 & 0.17678 & 0.16053 & 0.17855\tabularnewline
\bottomrule
\end{tabular}
\par\end{centering}
\end{table}

\begin{table*}[htbp]
\caption{Reprojection error and intrinsic parameter comparison for the experiment
with 10 images of the Checkerboard scene. For each statistic, the
top row corresponds to the value for cameras with $f=50$ mm (equivalent
in 35 mm film) and the bottom row corresponds to cameras with $f=100$
mm. Data are given in pixels.}
\label{tab:CheckerboardIntrinsics10imgs}
\begin{centering}
\begin{tabular}{lcccc}
\toprule 
Method & AQC & DAQ & \textbf{SLCV} & Euc. BA\\
\midrule
Mean focal length $\alpha$ & 1822.9 & 1851.4 & 1806.2 & 1847.8\\
 & 3433.2 & 3579.3 & 3399.4 & 3555.7\\[0.5ex]
$\alpha$ standard deviation & 10.95 & 4.86 & 16.1 & 7.77\\
 & 36.15 & 13.33 & 29.21 & 4.53\\[0.5ex]
Mean pp $(u_{0,}v_{0})$ & (623.1, 463.9) & (638.1, 477.8) & (607.4, 468.9) & (620.4, 485.3)\\
 & (618.6, 446.8) & (654.5, 505.7) & (605.5, 455.4) & (588.0, 525.2)\\[0.5ex]
p.p. std. dev. & (16.85, 8.13) & (1.61, 4.96) & (27.7, 11.15) & (6.72, 8.71)\\
 & (34.44, 0.55) & (17.27, 10.91) & (53.18, 7.47) & (24.79, 6.36)\\[0.5ex]
 Rep. error & 0.65765 & 0.27181 & 0.27497 & 0.22386\\
\bottomrule
\end{tabular}
\par\end{centering}
\end{table*}

Table \ref{tab:CheckerboardIntrinsics10imgs} shows the results for the
experiment with \emph{10 images}. In this case, it is possible to
compare the SLCV and DAQ algorithms to a technique based on the AQC
\cite{Valdes04}.
A Euclidean bundle adjustment with enforcement of the square-pixel condition is performed after the
Euclidean calibration with each of the four compared techniques. The
solution obtained (regardless of the AQC, DAQ or SLCV initialization),
is given in the last column of Table
\ref{tab:CheckerboardIntrinsics10imgs}.

If the image resolution and the size of the CCD of the camera are
known, it is possible to convert the focal length from mm to pixels.
For this dataset, an equivalent focal length of $f=50$ mm translates
to $\alpha=1850$ pixels, which is very close to the values obtained
by the different algorithms tested (first row of Table \ref{tab:CheckerboardIntrinsics10imgs}). 

Because the DAQ algorithm yields good estimates of the intrinsic parameters
(with small dispersion and reprojection error), we may conclude that
the principal point (p.p.) of the camera is close to the center of
the image. This observation is also supported by the estimates of
the p.p. due to the other algorithms. All three autocalibration methods
show a strong agreement with the three-homography calibration algorithm
in \cite{ronda08} and \cite[p.~211]{Hartley-Zisserman}.

\begin{figure}
\begin{centering}
\includegraphics[width=0.49\columnwidth]{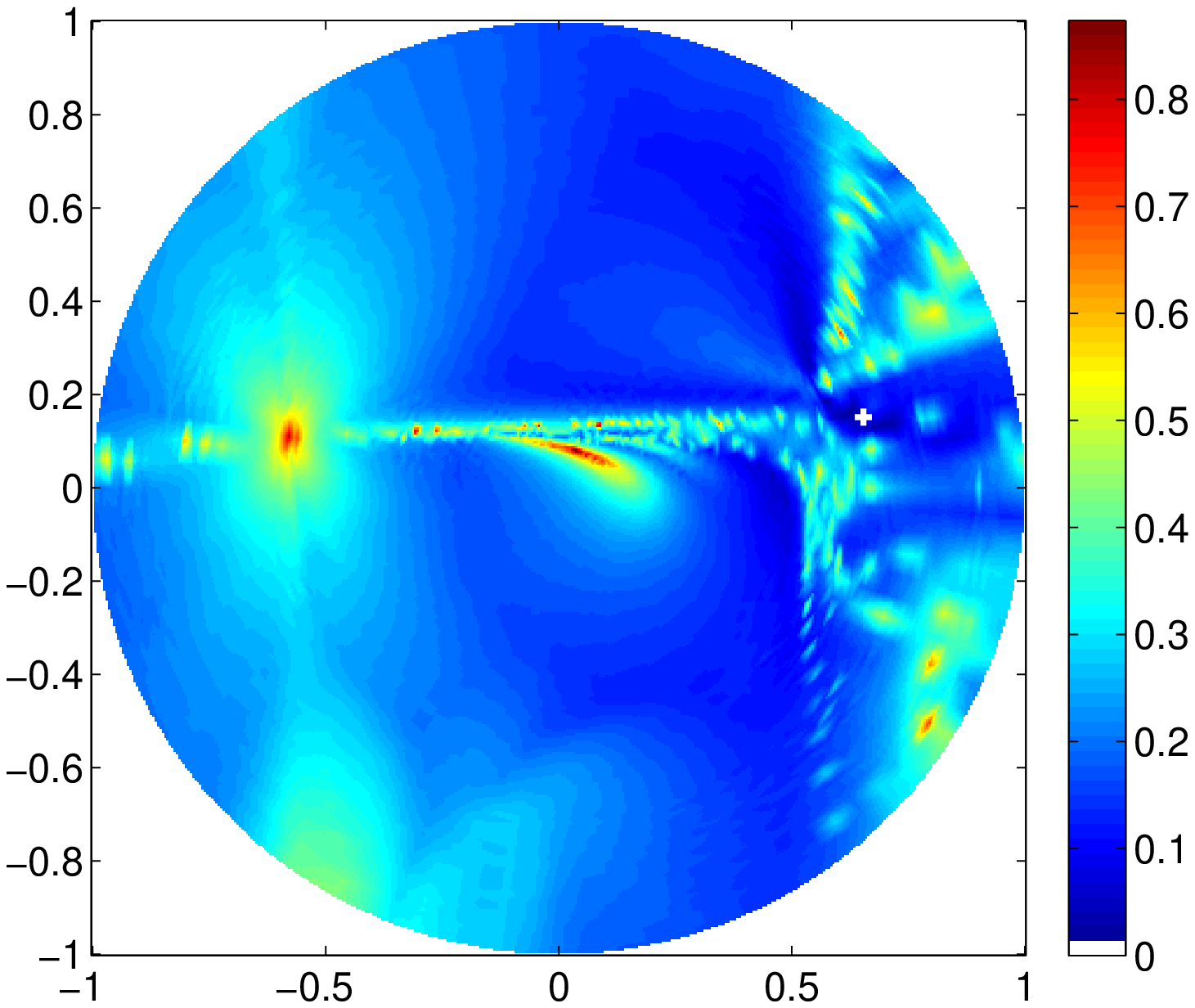}
\includegraphics[width=0.49\columnwidth]{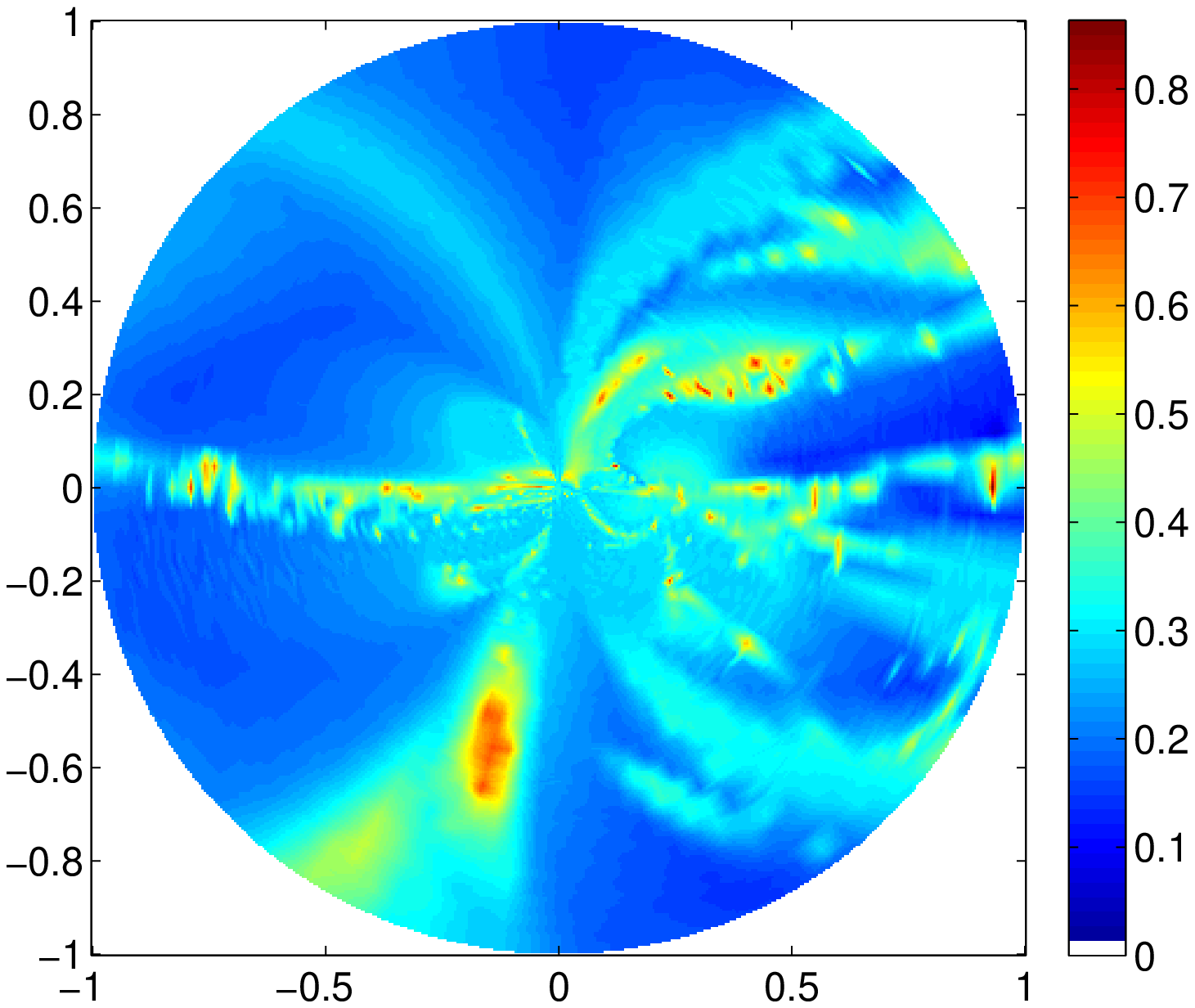}
\caption{Experiment with 10 images of the Checkerboard dataset: plots of the
sampled cost function. Left: values at the unit disc of the complex
plane, $|z|\leq1$. Right: values at the complement of the unit disc,
at positions $1/\bar{z}$. The white cross in the left plot marks
the location of the minimum.}
\label{fig:CheckerboardDiscs10imgs}
\end{centering}
\end{figure}

\begin{figure}
\begin{centering}
\includegraphics[width=\columnwidth]{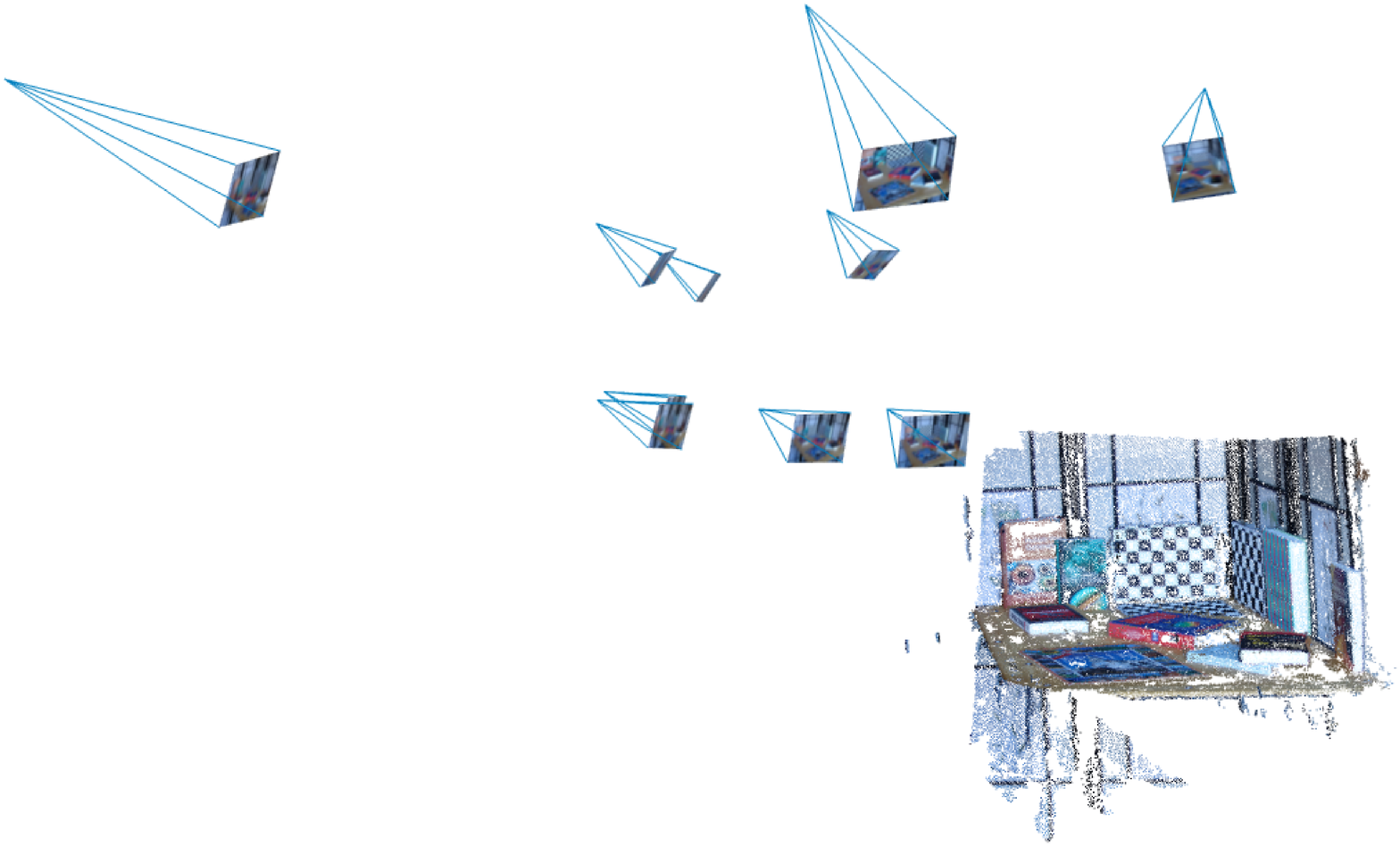}\medskip{}
\includegraphics[width=\columnwidth]{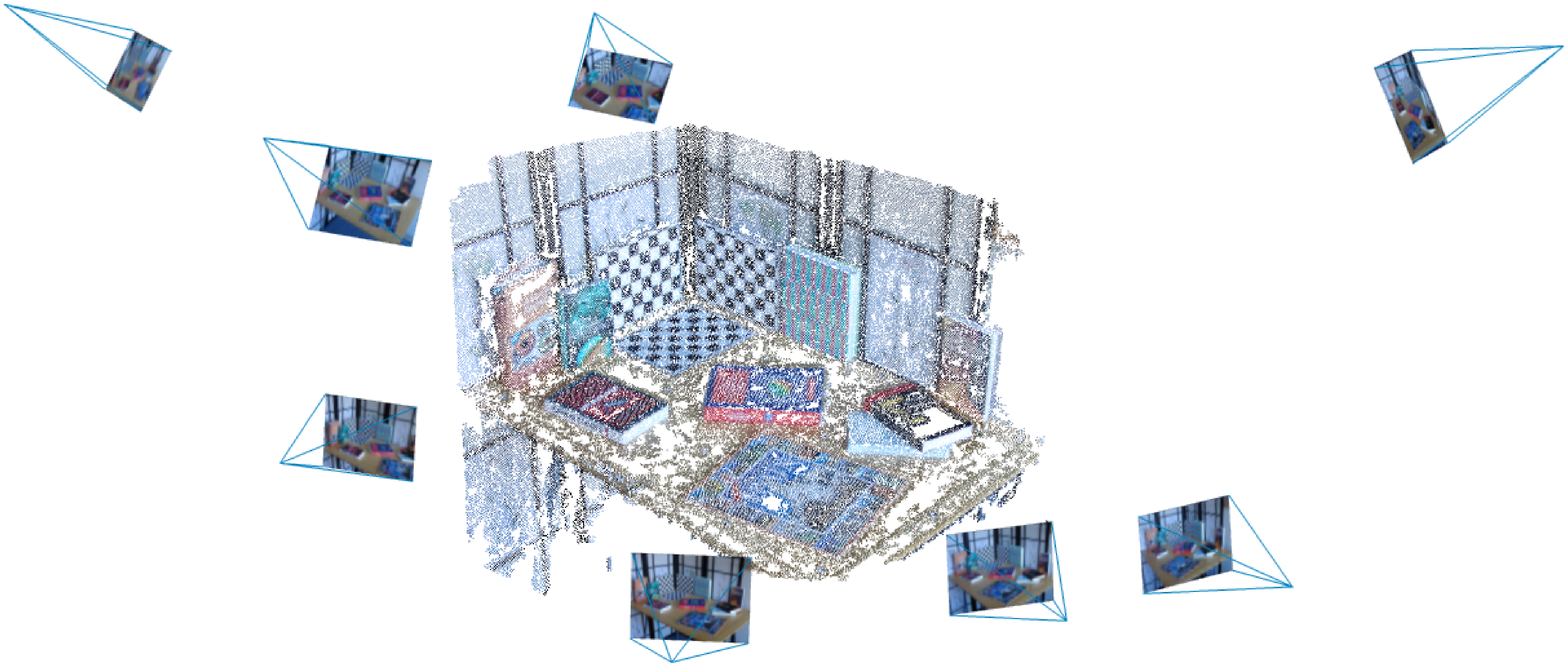}
\caption{Experiment with 10 images of the Checkerboard dataset: reconstructed
3D scene.}
\label{fig:CheckerboardScene10imgs}
\end{centering}
\end{figure}

Euclidean calibration by means of the developed technique (SLCV) provides
competitive results with respect to the other methods (DAQ or AQC), 
but with a slightly bigger dispersion around the mean values.
In normalized coordinates, the weights used in \eqref{eq:C0def} to
measure the goodness of fit of the IACs in Algorithm~\ref{alg:Algorithm2Upgrading} are $\gamma_{1}=\gamma_{2}=\gamma_{3}=\gamma_{4}=1$.
The sampling parameters in Algorithm~\ref{alg:Algorithm2Upgrading} are $M=N=100$. Figure~\ref{fig:CheckerboardDiscs10imgs}
shows the sampled cost function. 
Figure~\ref{fig:CheckerboardScene10imgs} shows two views of the densely reconstructed scene corresponding
to the Euclidean calibration by means of the SLCV algorithm. The dense
reconstruction was obtained by feeding the images and the Euclidean
calibration of the cameras to the Patch-based Multi-view Stereo Software (PMVS)~\cite{FurukawaP07}.

\begin{figure*}
\begin{centering}
\includegraphics[width=\linewidth]{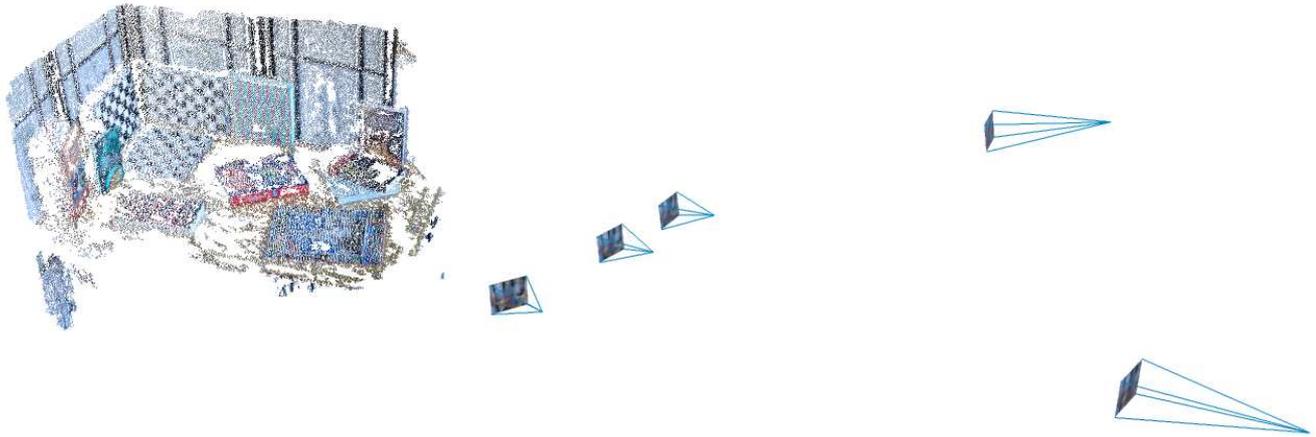}
\caption{Experiment with 5 images of the Checkerboard dataset: reconstructed 3D scene.}
\label{fig:CheckerboardScene05imgs}
\end{centering}
\end{figure*}

\begin{table*}
\caption{Reprojection error and intrinsic parameter comparison for experiment
with 5 images of the Checkerboard dataset. Same notation as in Table
\ref{tab:CheckerboardIntrinsics10imgs}. Both DAQ and SLCV methods yield good results.}
\label{tab:CheckerboardIntrinsics5imgs}
\begin{centering}
\begin{tabular}{lccc}
\toprule
Method & DAQ & \textbf{SLCV} & Euc. BA \tabularnewline
\midrule
Mean focal length $\alpha$ & 1833.3 & 1855.2 & 1883.1 \\
 & 3511.2 & 3602.1 & 3662.9\\[0.5ex]
$\alpha$ standard deviation & 0.46 & 3.84 & 4.78\\
 & 16.24 & 7.90 & 26.15\\[0.5ex]
Mean p.p. $(u_{0,}v_{0})$  & (637.4, 478.9) & (624.5, 477.2) & (620.0, 500.2) \\
 & (664.2, 493.3) & (618.0, 492.4) & (588.8, 569.5) \\[0.5ex]
p.p. std. dev.  & (0.72, 2.12) & (1.41, 1.48) & (8.72, 0.68) \\
 & (28.61, 15.91) & (31.53, 20.03) & (12.3, 15.4) \\[0.5ex]
Rep. error & 0.26204 & 0.20134 & 0.17689\\
\bottomrule
\end{tabular}
\par\end{centering}
\end{table*}

\begin{table*}
\caption{Reprojection error and intrinsic parameter comparison for experiment
with 5 expanded images of the Checkerboard dataset. Same notation as in Table
\ref{tab:CheckerboardIntrinsics10imgs}. 
The SLCV method clearly outperforms the DAQ method if the principal point is not near the image center.}
\label{tab:CheckerboardIntrinsics5imgsExpanded}
\begin{centering}
\begin{tabular}{lccc}
\toprule
Method & DAQ & \textbf{SLCV} & Euc. BA \\
\midrule
Mean focal length $\alpha$ & 3143.7 & 1827.2 & 1877.6 \\
 & 25499.9 & 3516.0 & 3643.56\\[0.5ex]
$\alpha$ standard deviation & 29.28 & 3.87 & 5.37\\
 & 20080.2 & 5.30 & 22.69\\[0.5ex]
Mean p.p. $(u_{0,}v_{0})$  & (801.3, 516.4) & (622.6, 478.5) & (617.5, 500.6) \\
 & (12192.1, -3639.6) & (626.0, 493.9) & (584.0, 571.8) \\[0.5ex]
p.p. std. dev.  & (261.1, 44.54) & (6.49, 0.92) & (7.54, 0.54) \\
 & (23167.2, 9700.0) & (43.46, 7.4) & (15.97, 9.37) \\[0.5ex]
Rep. error & 3324.95 & 0.20725 & 0.17707\\
\bottomrule
\end{tabular}
\par\end{centering}
\end{table*}

Next, an experiment with a subset of \emph{5 images} is carried out.
To account for varying parameters, three of the images correspond
to a focal length $f=50$ mm and the remaining two have $f=100$ mm.
Table~\ref{tab:CheckerboardProjective} shows the parameters of the
projective reconstruction. Table~\ref{tab:CheckerboardIntrinsics5imgs}
compares the Euclidean upgrading given by the DAQ and the SLCV algorithms. 
Both initializations converge to the same solution after 
Euclidean bundle adjustment (last column of Table~\ref{tab:CheckerboardIntrinsics5imgs}).
The SLCV method yields similar results to those of the DAQ method but requiring fewer
equations per camera. The reconstructed scene corresponding to the
Euclidean calibration by means of the SLCV algorithm is shown in Fig.~\ref{fig:CheckerboardScene05imgs}. 
The SLCV method provides a sensible initialization to Euclidean bundle adjustment. 

Next, we demonstrate the good performance of the developed method
in case of decentered principal point. To do so, the previous 5 images of
the Checkerboard dataset are extended to $1600\times1200$ pixels from
the upper-left corner. The principal point is, as seen in Table~\ref{tab:CheckerboardIntrinsics10imgs},
near the point with coordinates $(640,480)$ pixels, which significantly
differs from the new image center at $(800,600)$ pixels.
A projective reconstruction of the scene is obtained, with 1197 \mbox{3-D} points, 4463 image projections and a
reprojection error of 0.17699 pixels.
Table~\ref{tab:CheckerboardIntrinsics5imgsExpanded} compares the
Euclidean upgrading by the two autocalibration methods in Table~\ref{tab:CheckerboardIntrinsics5imgs}.
Because the hypothesis of known principal point (e.g. at the image center) is
not satisfied, the DAQ method performs poorly. 
The SLCV method, however, yields the similar good results as in Table~\ref{tab:CheckerboardIntrinsics5imgs}
because it solely relies on the square-pixel constraint. 
The last column of Table \ref{tab:CheckerboardIntrinsics5imgsExpanded} shows the
result after refining the SLCV Euclidean calibration by bundle adjustment.

\subsection{Outdoor dataset}
\label{sec:OutdoorScene}

As an additional experiment, 16 images of the Plaza de la Villa in Madrid 
(see Fig. \ref{fig:PlazaDeLaVillaImages}) were acquired with an Olympus E-620 digital camera
at a resolution of $1280\times960$ pixels. 
The focal length was set to $f=50$ mm ($\alpha=1850$ pixels) in half of the images and
to $f=70$ mm ($\alpha=2590$ pixels) in the other half. Reconstructions
were carried out with all images and with a subset of 5 images (with
different focal lengths). The parameters of the projective reconstructions
are summarized in Table \ref{tab:CheckerboardProjective}. Figures
\ref{fig:PlazaDeLaVillaScene16imgs} and \ref{fig:PlazaDeLaVillaScene05imgs}
show reconstructions of the scene with 16 and 5 images, respectively,
obtained by Euclidean upgrading with the SLCV method. 

\begin{figure}
\begin{centering}
\includegraphics[width=0.49\columnwidth]{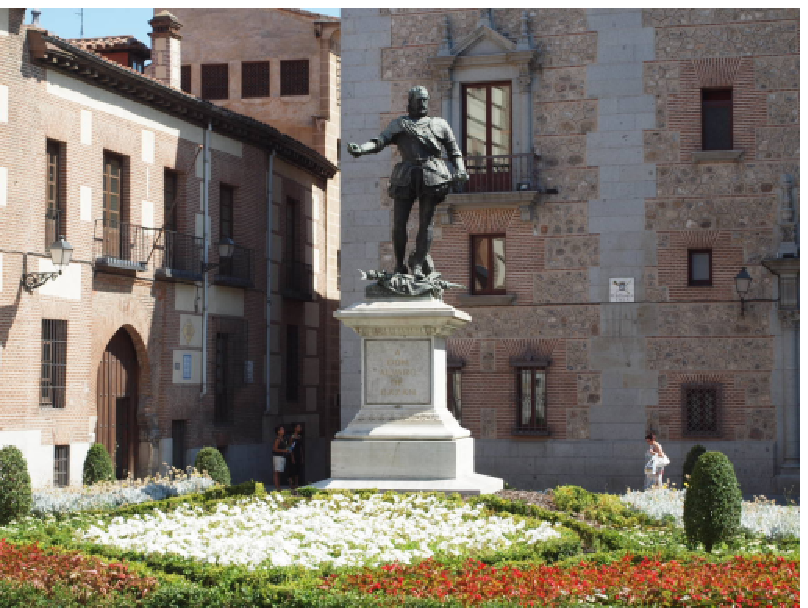}
\includegraphics[width=0.49\columnwidth]{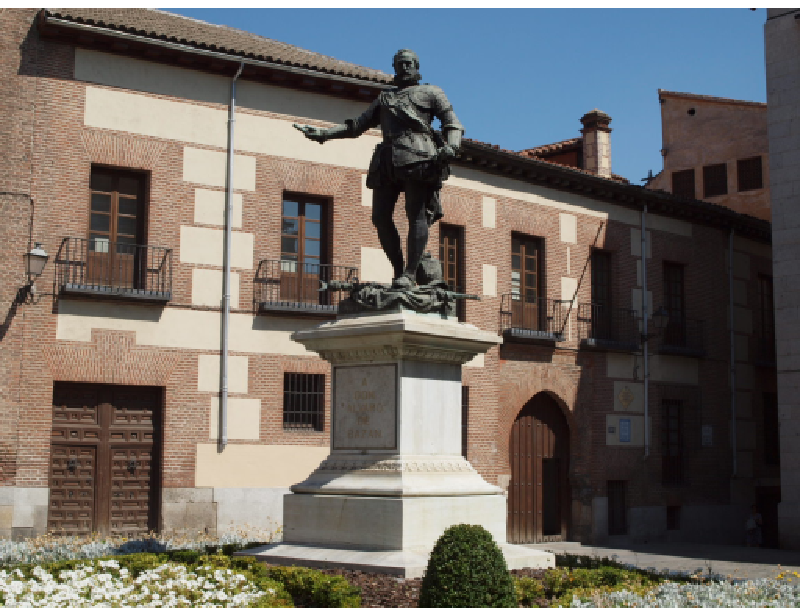}
\caption{Two images of the outdoor (\emph{Plaza de la Villa}) scene.}
\label{fig:PlazaDeLaVillaImages}
\end{centering}
\end{figure}

\begin{figure}
\begin{centering}
\includegraphics[width=\columnwidth]{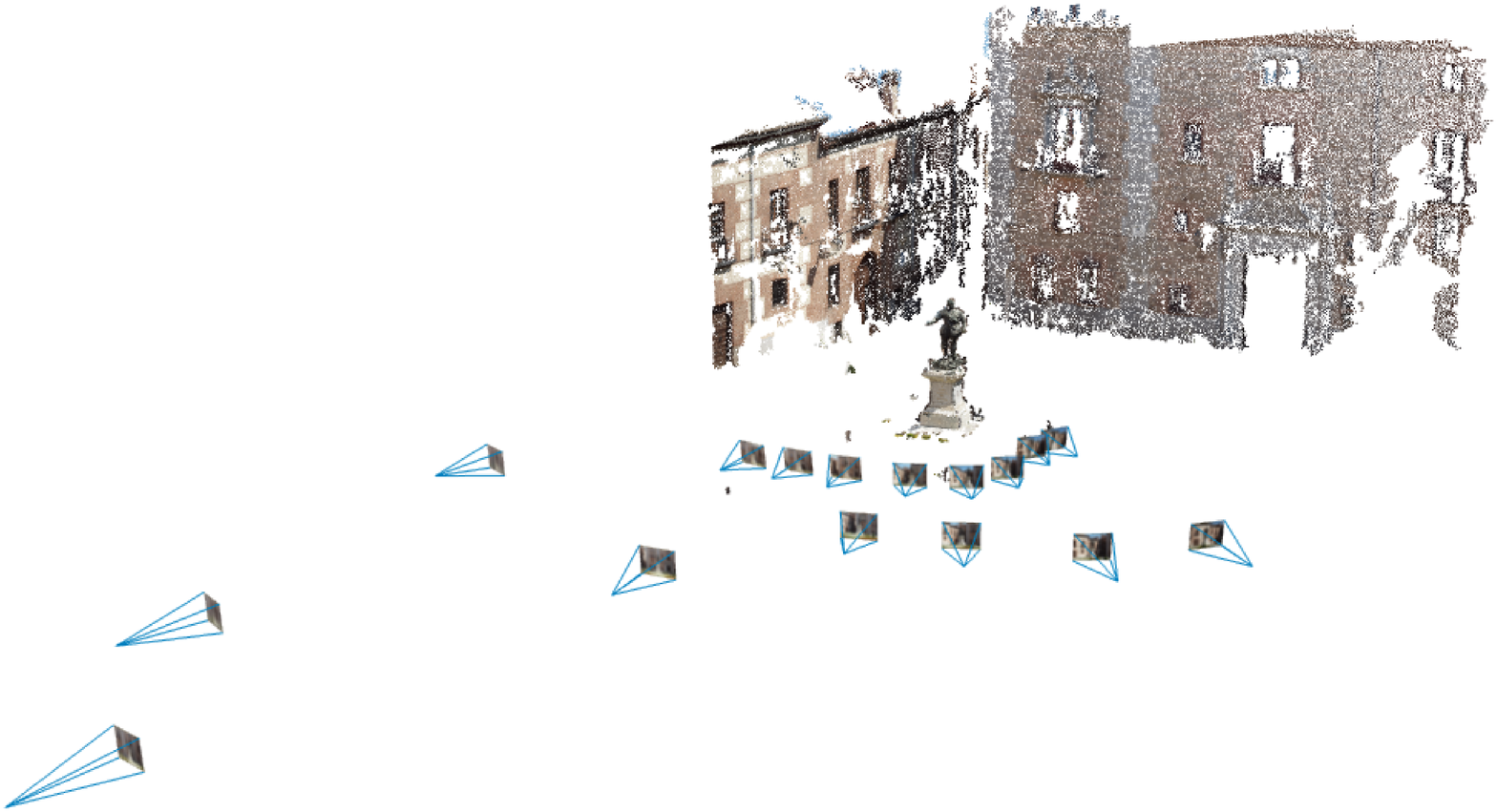}
\includegraphics[width=\columnwidth]{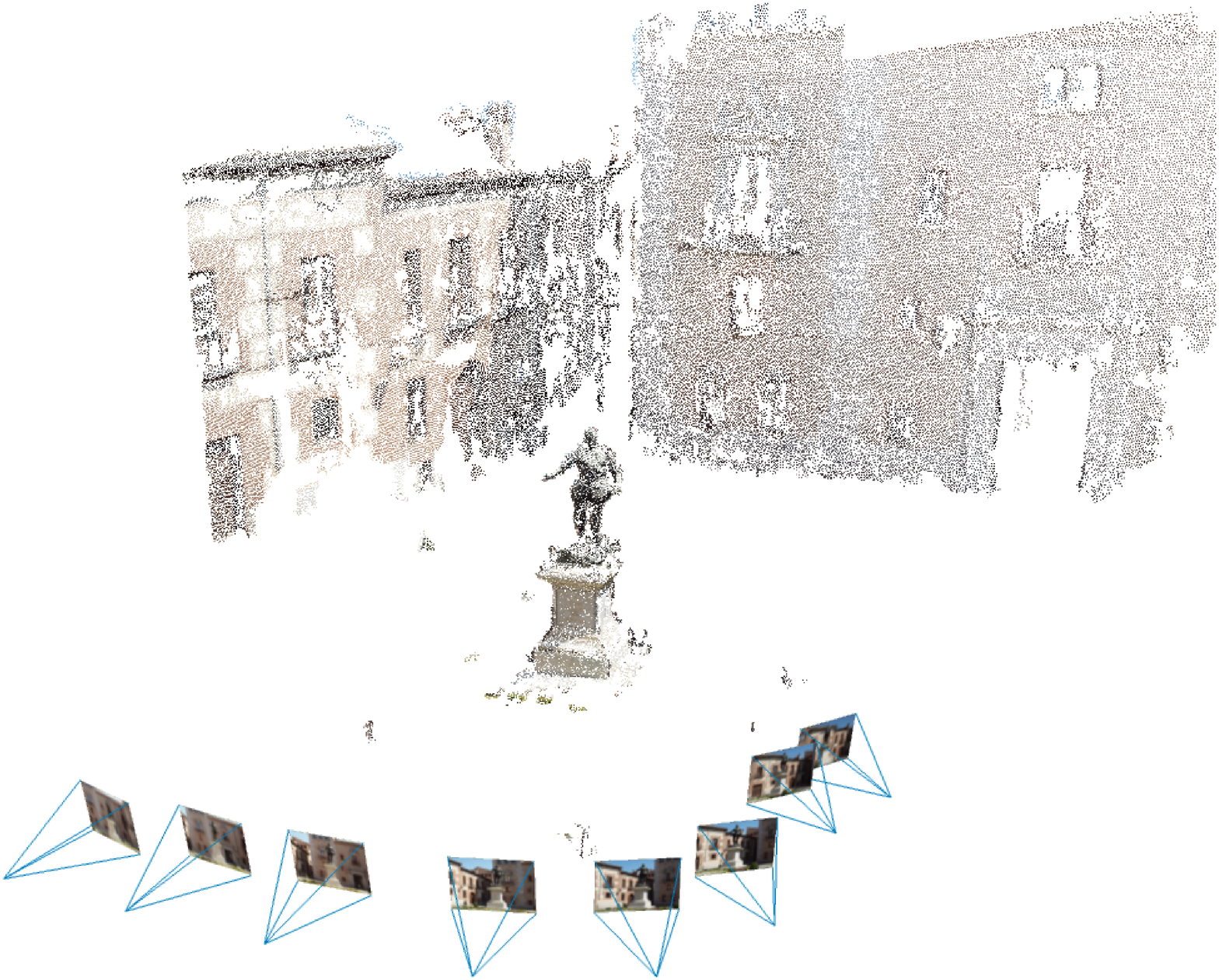}
\caption{Experiment with 16 images of the outdoor dataset: reconstructed
3D scene.}
\label{fig:PlazaDeLaVillaScene16imgs}
\end{centering}
\end{figure}

\begin{figure*}
\begin{centering}
\includegraphics[width=0.89\linewidth]{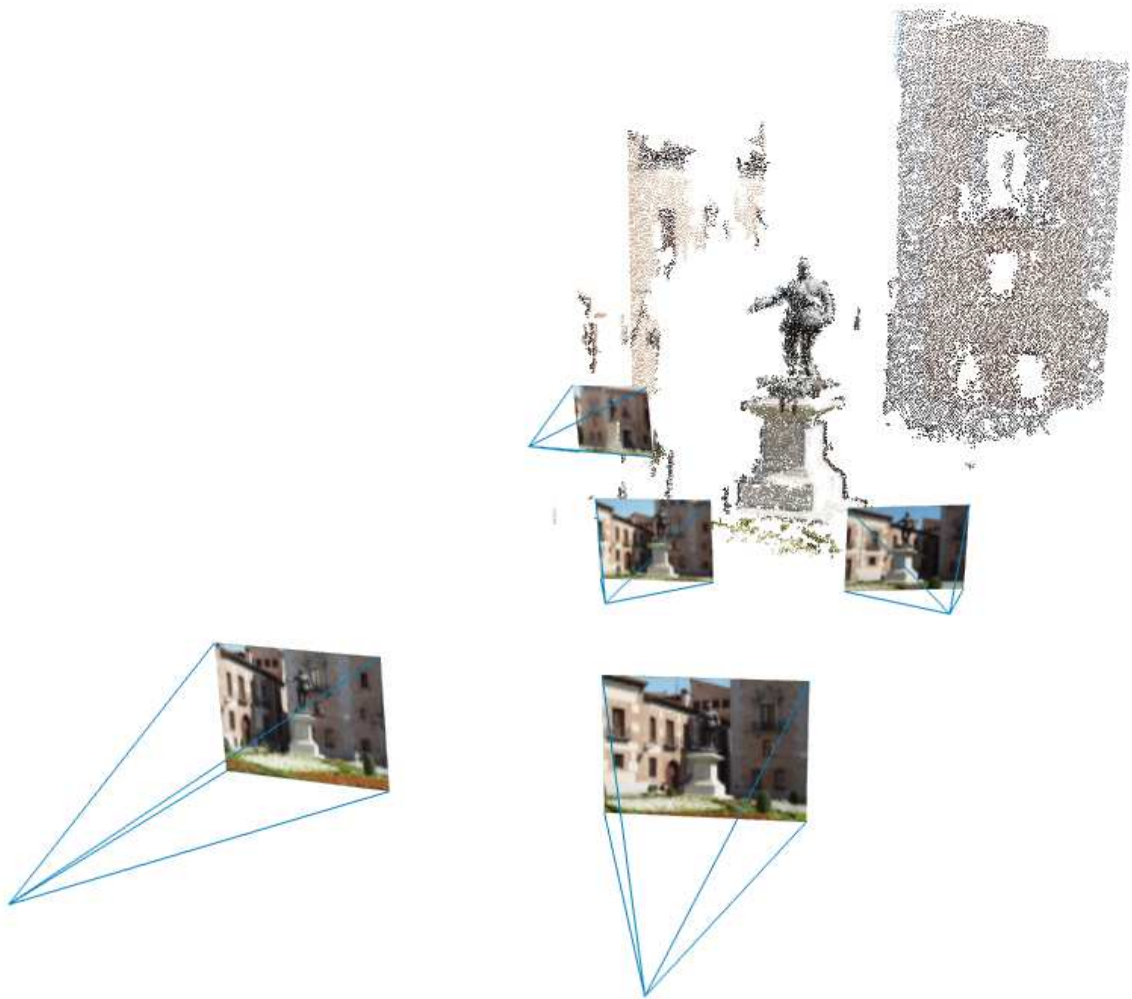}
\caption{Experiment with 5 images of the outdoor dataset: reconstructed
3D scene.}
\label{fig:PlazaDeLaVillaScene05imgs}
\end{centering}
\end{figure*}

\section{Discussion}
\label{sec:discussion}
The search of the plane at infinity in a projectively reconstructed
scene obtained with five or more square-pixel cameras is in principle a three-dimensional optimization problem. The SLCV concept allows to convert it into a two-dimensional optimization problem by restricting the search to the set of candidate planes that are compatible with the restrictions provided by a given subset of three of the cameras.

The empirical comparison of the SLCV algorithm
  performance with alternative 3D-search based algorithms presented in
  the results section points out two
  facts: First, the SLCV algorithm does provide valid results with
  different sets of real image data without previous preprocessing of
  the projective reconstruction, while for the tested 3D-search algorithms a
  previous quasi-affine upgrading seems to be mandatory.  Second, the
  availability of this quasi-affine upgrading can be effective to the
  point of compensating for the larger dimension of the search space.

Summing up these to facts, we may conclude 
  that the choice between the SLCV unrestricted 2D search and a
  cheirality constrained 3D search depends on the effectiveness of
  these constraints to narrow the search area, which in turn depends
  on the scene contents.

Although cheirality constraints have not yet been
  considered in the proposed algorithm, nothing prevents this
  integration and therefore it constitutes an interesting topic for
  future research. Two approaches are natural: If the convex hull of
  the camera centers and scene points intersects one of the camera
  principal planes, cheirality constraints emerge by restricting the
  search to the lines in this principal plane that do not intersect
  the convex hull. In any case, the search can be performed on the set
  of lines of any plane of space. For a general plane, each line will
  require the solution of a fifth-degree polynomial, while for a plane
  contained in the SLCV it suffices a fourth-degree polynomial that,
  as is well known, can be solved using a closed-form expression.
  Result~\ref{res:included_pencils} provides a wide source of such
  planes that can be used for this purpose.

\section{Conclusions}
\label{sec:conclusions}

In this paper we have proposed an algorithm to obtain a Euclidean
reconstruction from the minimum possible number of cameras with known
pixel shape but otherwise varying parameters, i.e., five cameras.
To this purpose we have introduced, as our main tool, the geometric
object given by the variety of conics intersecting six given spatial
lines simultaneously (the six-line conic variety, SLCV). We have presented
an independent and self-contained treatment including a procedure
for the explicit computation of the equation of the SLCV defined by
three cameras.

While the SLCV for six lines in generic position is, in general, a
surface of degree 8, we have shown that this degree can be reduced to
5 in the case of the three pairs of isotropic lines of three finite
square-pixel cameras. We have seen that a direct application of this
result is the obtainment of the candidate planes at infinity in case two
points at infinity are known.

We have seen that the fifth-degree SLCV has three singularities of
multiplicity three. We have used this fact to obtain a computationally
efficient parameterization of the SLCV that, if we have some additional
data, permits the obtainment of the plane at infinity by means of
a two-dimensional search. An algorithm has been proposed for the case
in which two or more additional cameras are available.

Experiments with real images for the autocalibration of scenes with 5,
10 and 16 cameras have been given, showing the good performance of the
SLCV technique.  Particularly, in the case of fewer than 10 cameras,
so that the AQC cannot be used, the SLCV method overcomes the
limitation of the DAQ based technique for camera autocalibration with
varying parameters, i.e., the need for known principal
point. We have also included a comparison with two
  other algorithms based on 3D-search on the space of planes.  Thus,
the developed method seems to be a feasible approach to solve the autocalibration
problem in the above situation without requiring previous
initialization.  Furthermore, we have shown that the SLCV method is
also a sound alternative to other approaches that require 10 or more
cameras or some knowledge of the principal point.

\section*{Acknowledgment}
The authors thank the anonymous reviewers for
their helpful comments, in particular for suggesting the comparisons of
the SLCV algorithm with the 3D search algorithms, and for pointing
out references~\cite{Finsterwalder} and~\cite{Sturm11}.

\appendix

\section{Proof of Result~\ref{th:determinant_characterizes}}
\normalsize
\label{sec:Proof_determinant_characterizes}
\begin{proof}
  First note that if $D(\bpi,\ba_1,\ldots,\ba_4)=0$, the points
  $\mL_{i}\bpi,\ba_j$ lie on a quadric $Q$. By construction, the
  points $\mL_{i}\bpi$ are also on the plane $\bpi$,  and therefore they
  lie on the conic $\bpi\cap Q$.

Conversely, let us suppose that the points $\mL_{i}\bpi$ lie on a
conic $C$ contained in $\bpi$. Let us choose coordinates $(x,y,z,t)$ such that the plane
$\bpi$ is given by the equation $t=0$. Let $C(x,y,z)=t=0$ be the
equations of $C$. Any quadric $Q$ containing $C$
has an equation of the form
\begin{displaymath}
  \alpha\,C(x,y,z) + t(ax+by+cz+dt)=0.
\end{displaymath}
Given four points $\ba_j=(x_j,y_j,z_j,t_j)^\top$, 
it is always possible to find some quadric $Q$ through them
because the points lead to a linear homogeneous system of
four equations in the five unknowns $a,b,c,d,\alpha$ which always
admits at least one non-trivial solution.
Since the $10$ points $\mL_{i}\bpi,\ba_j$ lie on $Q$, 
we conclude that $D(\bpi,\ba_1,\ldots,\ba_4)=0$.
\qed\end{proof}

\section{Proof of Result~\ref{th:determinant_factorizes}}
\label{sec:Proof_determinant_factorizes}
\begin{proof}
Any plane $\bpi$ through the point $\ba_{1}$ is a
  trivial solution of $D=0$, since the seven points
  $\mL_{1}\bpi,\ldots,\mL_{6}\bpi,\ba_{1}$ would lie on the plane
  $\bpi$ and therefore the ten points lie on the quadric given by
  $\bpi$ and the plane formed by points $\ba_{2},\ba_{3},\ba_{4}$.
  Consequently, the  planes through  $\ba_{1}$ produce
  a linear factor $\bpi^{\top}\ba_{1}$ of $D$ and so do the planes
  through each of the three remaining points $\ba_j$.

  Let us denote $\mA=(\ba_{1},\ldots,\ba_{4})$. Next we show that $\det\mA$ also divides $D$.
  Since the determinant $\det\mA$ is irreducible~\cite[p.~176]{Bocher} when
  regarded as a polynomial in the coordinates of the points
  $\ba_j$, it is enough to show that $D$ vanishes whenever
  $\det\mA=0$, i.e., when the points $\ba_{j}$ are
  contained in some plane $\bpi'$.  In such case, the ten
  points $\mL_{i}\bpi$, $i=1,\ldots,6$ and $\ba_{j}$, $j=1,\ldots,4$,
  lie in the degenerate quadric $\bpi\cdot\bpi'$ 
 so that $D$ cancels. We have
  therefore the required
  factorization~\eqref{eq:determinant_factorizes}.  

  Let us now show that $F$ does not depend on the variables
  $\ba_j$. The degree of the variables $\ba_{j}$ on the left hand side
  of~\eqref{eq:determinant_factorizes} is two because the vectors
  $\nu_2(\ba_j)$ are homogeneous quadratic in the entries of $\ba_j$
  and each of the $\nu_2(\ba_j)$ appears only once as a column of the
  determinant. On the right hand side
  of~\eqref{eq:determinant_factorizes} we have the factor $\det\mA$,
  which is homogeneous of degree one in each $\ba_j$, and the factors
  $\bpi^\top \ba_j$, which are also homogeneous of degree
  one. Therefore the remaining factor $F$ does not depend on the
  $\ba_{j}$ since otherwise there would be a mismatch between the
  degrees of both sides of~\eqref{eq:determinant_factorizes}. 

  Finally, let us check that the equation $F(\bpi)=0$ characterizes
  the planes intersecting the six lines in points of a conic. Now that
  we have proven factorization~\eqref{eq:determinant_factorizes}, it
  is trivial that if $F(\bpi)=0$ then $D(\bpi, \ba_1,\ldots,\ba_4)=0$
  and, by Result~\ref{th:determinant_characterizes}, $\bpi$ is a 
  plane intersecting the six lines in points of a conic. 
  Conversely, if $\bpi$ intersects the six lines in
  points of a conic, the determinant $D$ vanishes for any $\ba_j$. In
  particular, choosing any non-coplanar points $\ba_j$ not in $\bpi$, so
  that $\det\mA\not=0\not=\bpi^\top\ba_j$, we conclude that
  $F(\bpi)=0$.
\qed\end{proof}

\section{Proof of Result~\ref{res:sing_planes}}
\label{sec:Proof_sing_planes}
\begin{proof} Let $l_{i},\bar{l}_{i}$ note the pair of lines through the optical center $\bC_{i}$ and 
let $\mL_{i}$, $\bar{\mL}_{i}$ be their Pl\"ucker matrices. 
Also, let the corresponding principal planes be $\bpi_{i}$, $i=1,2,3$. 
Using Result~\ref{th:determinant_factorizes} we have that 
\[
\begin{split}D & =\det\left(\nu_{2}(\mL_{1}\bpi),\ldots,\nu_{2}(\bar{\mL}_{3}\bpi),\nu_{2}(\ba_{1}),\ldots,\nu_{2}(\ba_{4})\right)\\
 & =\det(\ba_{1},\ldots,\ba_{4})\,(\bpi^{\top}\ba_{1})\cdots(\bpi^{\top}\ba_{4})\,(\bpi^{\top}\bC_{1})(\bpi^{\top}\bC_{2})\\
 &\qquad (\bpi^{\top}\bC_{3})\, G(\bpi).
\end{split}
\]
Assuming that $\bpi_{1}$ is a generic principal plane, let the candidate planes be parameterized as $\bpi=\lambda\bpi_{1}+\mu\bxi$.
Since $l_{1},\bar{l}_{1}$ are contained in $\bpi_{1}$ we have $\mL_{1}\bpi_{1}=\bar{\mL}_{1}\bpi_{1}=\mathbf{0}$
and therefore $\mL_{1}\bpi=\mu\mL_{1}\bxi$ and $\bar{\mL}_{1}\bpi=\mu\bar{\mL}_{1}\bxi$.
Consequently 
\[
D(\lambda,\mu)=\mu^{4}\det\left(\nu_{2}(\mL_{1}\bxi),\nu_{2}(\bar{\mL}_{1}\bxi),\nu_{2}(\mL_{2}\bpi),\ldots,\nu_{2}(\ba_{4})\right).
\]
 On the other hand, since $\bC_{1}$ in $\bpi_{1}$, we have $\bC_{1}^{\top}\bpi=\mu\bC_{1}^{\top}\bxi$,
so 
\[
\begin{split} & \mu^{4}\det\left(\nu_{2}(\mL_{1}\bxi),\nu_{2}(\bar{\mL}_{1}\bxi),\nu_{2}(\mL_{2}\bpi),\ldots,\nu_{2}(\ba_{4})\right)\\
 & =\mu\det(\ba_{1},\ldots,\ba_{4})\,(\bpi^{\top}\ba_{1})\cdots(\bpi^{\top}\ba_{4})\,(\bxi^{\top}\bC_{1})(\bpi^{\top}\bC_{2})\\
 &\qquad (\bpi^{\top}\bC_{3})\, G(\bpi).
\end{split}
\]
 Using the genericity of $\bpi_{1}$ and choosing conveniently the
points $\ba_j$ we have that $\bC_{2},\bC_{3},\ba_{j}$ not in $\bpi_{1}$
for $j=1,\ldots,4$, i.e., $\mu$ does not divide any linear factor
besides $\bpi^{\top}\bC_{1}$. Therefore $\mu^{3}$ divides $G(\lambda\bpi_{1}+\mu\bxi)$
and so $\bpi_{1}$ is a singular point of $G$ of multiplicity three.

\qed\end{proof}

\section{Calculation of the Isotropic Lines}
\label{sec:isotropic_lines}
The isotropic lines are the back-projection of the cyclic points at
infinity $(1, \pm i, 0)^\top$. Let us see how we can compute their
Pl\"ucker matrices from the rows of the corresponding projection
matrix $\mP$. In a projective reference in which $\mP =
(\bp_1,\bp_2,\bp_3)^\top$, the back-projection of the cyclic points
is given by those $\bX$ in $\bP^3$ such that
\begin{displaymath}
  \mP\bX \sim (1,\pm i, 0)^\top.
\end{displaymath}
Therefore the cross-product
\begin{displaymath}
  \mP\bX\times  (1,\pm i, 0)^\top=0,
\end{displaymath}
or, equivalently, $\bX$ satisfies the equations
\begin{displaymath}
  \bp_3^\top\bX =0=(\bp_2\pm i\bp_1)^\top\bX.
\end{displaymath}
Hence, the isotropic lines are defined by the intersection of the
planes $\bp_3$ and $\bp_2\pm i\bp_1$. Finally, their Pl\"ucker
matrices are given by
\begin{displaymath}
  \mL = \mM(\bp_3,\bp_2+ i\bp_1)^\ast,\quad \bar\mL = \mM(\bp_3,\bp_2- i\bp_1)^\ast 
\end{displaymath}
where $\mM$ and the $\ast$ operator were defined in Section~\ref{subsec:plucker_coordinates}.

\section{Computation of Polynomial $H_{0}$ }
\label{sec:H0coeffs}

The polynomial $H_{0}(\lambda,\mu)=A_{0}\lambda^{2}+B_{0}\lambda\mu+C_{0}\mu^{2}$
in equation (\ref{eq:H_0}) can be computed as follows. Performing
the coordinate change (\ref{eq:mapsto}), the Pl\"ucker matrices of
the lines $l_{i}$ are 
\[
\begin{split}\mL_{1} & =\bC_{1}\bq^{\top}-\bq\bC_{1}^{\top}=\left(\begin{array}{cccc}
0 & 0 & 0 & -1\\
\noalign{\medskip}0 & 0 & 0 & -i\\
\noalign{\medskip}0 & 0 & 0 & 0\\
\noalign{\medskip}1 & i & 0 & 0
\end{array}\right),\\
\mL_{2} & =\bC_{2}\bp_{2}^{\top}-\bp_{2}\bC_{2}^{\top}=\left(\begin{array}{cccc}
0 & 0 & -x_{{2}} & -x_{{2}}\\
\noalign{\medskip}0 & 0 & -y_{{2}} & -y_{{2}}\\
\noalign{\medskip}x_{{2}} & y_{{2}} & 0 & -z_{{2}}\\
\noalign{\medskip}x_{{2}} & y_{{2}} & z_{{2}} & 0
\end{array}\right),\\
\mL_{3} & =\bC_{3}\bp_{3}^{\top}-\bp_{3}\bC_{3}^{\top}=\left(\begin{array}{cccc}
0 & -x_{{3}} & x_{{3}} & -x_{{3}}\\
\noalign{\medskip}x_{{3}} & 0 & z_{{3}}+y_{{3}} & -y_{{3}}\\
\noalign{\medskip}-x_{{3}} & -y_{{3}}-z_{{3}} & 0 & -z_{{3}}\\
\noalign{\medskip}x_{{3}} & y_{{3}} & z_{{3}} & 0
\end{array}\right),
\end{split}
\]
 where the points $\bp_{i}=(x_{i},y_{i},z_{i},0)^{\top}$ are the
intersection of lines $l_{i}$ with the plane $\bpi_{4}=(0,0,0,1)^{\top}$.
The Pl\"ucker matrices $\bar{\mL}_{i}$ are just the complex conjugate
of the matrices $\mL_{i}$. Substituting in equation (\ref{eq:determinant_definition})
and factoring out the trivial linear factors we obtain the polynomial
$H_{0}(\alpha,\beta)$. 
Explicit formulas can be found in 
\protect\href{http://www.gti.ssr.upm.es/~jir/SLCV.html}{http://www.gti.ssr.upm.es/~jir/SLCV.html}

\bibliographystyle{spmpsci}      
\bibliography{SLCV-JMIV-postprint} 
\end{document}